%% file: notes.tex
 \documentclass[11pt]{article}

\usepackage{etoolbox}

\newtoggle{coltformat}
 \togglefalse{coltformat}

\newcommand{\colt}[1]{\iftoggle{coltformat}{#1}{}}
\newcommand{\arxiv}[1]{\iftoggle{coltformat}{}{#1}} 

\arxiv{

\input{arxiv_style}
}

\colt{

\title[Optimistic Rates for LLP]{Optimistic Rates for Learning from Label Proportions}

}

\arxiv{
\title{Optimistic Rates for Learning from Label Proportions}
}

\input{macros}  

\colt{
\usepackage{times}
\usepackage{parskip}
\usepackage{subcaption}
\usepackage{booktabs}
\usepackage{multirow}
}

\usepackage[suppress]{color_edits}
\addauthor{GL}{red}
\addauthor{LC}{green}
\hypersetup{colorlinks=true,allcolors=blue}

\colt{
\renewenvironment{proof}[1][]{\par\noindent{\bf #1\ }}{\hfill\BlackBox\\[2mm]}
}

\colt{
\coltauthor{%
 \Name{Gene Li} \Email{gene@ttic.edu}\\
 \addr Toyota Technological Institute at Chicago \\
 \Name{Lin Chen} \Email{linche@google.com}\\
 \addr Google Research \\
 \Name{Adel Javanmard} \Email{ajavanma@usc.edu}\\
\addr Marshall School of Business, University of Southern California\\
 \addr Google Research \\
 \Name{Vahab Mirrokni} \Email{mirrokni@google.com}\\
 \addr Google Research \\
}
}
\arxiv{
\author{%
  Gene Li$^{1}$\thanks{Part of this work was done while GL was an intern at Google Research.}\\
  \and
  Lin Chen$^{2}$ \\
  \and
  Adel Javanmard$^{2,3}$ \\
  \and 
  Vahab Mirrokni$^{2}$
  \and  
  $^{1}$TTIC, $^{2}$Google Research, $^{3}$University of Southern California 
}
}

\begin{document} 
\maketitle 

\begin{abstract}
We consider a weakly supervised learning problem called Learning from Label Proportions (LLP), where examples are grouped into ``bags'' and only the average label within each bag is revealed to the learner. We study various learning rules for LLP that achieve PAC learning guarantees for classification loss. We establish that the classical Empirical Proportional Risk Minimization (EPRM) learning rule \citep{yu2014learning} achieves fast rates under realizability, but EPRM and similar proportion matching learning rules can fail in the agnostic setting. We also show that (1) a debiased proportional square loss, as well as (2) a recently proposed EasyLLP learning rule \citep{busa2023easy} both achieve ``optimistic rates'' \citep{panchenko2002some}; in both the realizable and agnostic settings, their sample complexity is optimal (up to log factors) in terms of $\eps, \delta$, and VC dimension. 
\end{abstract}

\colt{
}

\input{intro} 
\input{proportion_matching}
\input{unbiased_square}

\input{easyllp}

\input{lower_bounds_main}
\arxiv{
\input{experiments}

}
\input{discussion}
\colt{
\acks{We thank Robert Istvan Busa-Fekete, Travis Dick, Claudio Gentile, Nathan Srebro, and Han Shao for helpful discussions. Adel Javanmard is supported in part by the NSF Award
DMS-2311024 and the Sloan fellowship in Mathematics.}
}
\arxiv{
\subsection*{Acknowledgements}
We thank Robert Istvan Busa-Fekete, Travis Dick, Claudio Gentile, Nathan Srebro, and Han Shao for helpful discussions. Adel Javanmard is supported in part by the NSF Award
DMS-2311024 and the Sloan fellowship in Mathematics.
}

\addcontentsline{toc}{section}{References}

\bibliography{notes.bbl}

\clearpage

\appendix
\input{failure_f_log}
\input{optimistic_rates}

\input{debiased_sq_proof}
\input{easyllp_proof}
\input{lower_bounds}
\input{technical}
\colt{\input{experiments}}
\input{app_experiments}

\end{document}

%% file: arxiv_style.tex
\usepackage[T1]{fontenc}    
\usepackage{url} 
\usepackage{booktabs}       
\usepackage{nicefrac}       
\usepackage{algorithm}
\usepackage{algorithmicx}
\usepackage{algpseudocode}

\usepackage{multirow}
\usepackage{multicol}

\usepackage{parskip}

\renewcommand{\paragraph}[1]{
	\textbf{#1}~~}

 \usepackage[letterpaper, left=1in, right=1in, top=1in, bottom=1in]{geometry}

\usepackage[dvipsnames]{xcolor}
\usepackage[colorlinks=true, linkcolor=blue!70!black, citecolor=blue!70!black]{hyperref}
\usepackage{microtype}

\usepackage{natbib}
\bibliographystyle{plainnat}
\bibpunct{(}{)}{;}{a}{,}{,}

\usepackage{amsthm}
\usepackage{mathtools}
\usepackage{amsmath}
\usepackage{bbm}
\usepackage{amsfonts}
\usepackage{amssymb}
\usepackage{natbib}
\let\vec\undefined
\usepackage{MnSymbol} 
\usepackage{xpatch}
\usepackage{pifont}
\usepackage{array}
\usepackage{booktabs}
\usepackage{blindtext}
\usepackage{caption}

\usepackage{subcaption}

\newtheorem{corollary}{Corollary}
\newtheorem{proposition}{Proposition}
{\newtheorem{lemma}{Lemma}} 
{\newtheorem{theorem}{Theorem}}
\newtheorem{definition}{Definition}

%% file: macros.tex
\usepackage{bbm}

\newcommand\numberthis{\addtocounter{equation}{1}\tag{\theequation}} 
\allowdisplaybreaks

\usepackage{mathtools}

\DeclarePairedDelimiter{\abs}{\lvert}{\rvert} 
\DeclarePairedDelimiter{\brk}{[}{]}
\DeclarePairedDelimiter{\crl}{\{}{\}}
\DeclarePairedDelimiter{\prn}{(}{)}
\DeclarePairedDelimiter{\nrm}{\|}{\|}

\let\Pr\undefined
 
\DeclareMathOperator{\En}{\mathbb{E}}

\DeclareMathOperator{\Pr}{\bbP}

\DeclareMathOperator*{\argmin}{argmin} 

\newcommand{\ind}[1]{\mathbbm{1}\crl*{#1}}    
\newcommand{\indd}[1]{\mathbbm{1}\crl{#1}}    
\newcommand{\eps}{\varepsilon}

\newcommand{\wt}[1]{\widetilde{#1}}
\newcommand{\wh}[1]{\widehat{#1}}

\def\ddefloop#1{\ifx\ddefloop#1\else\ddef{#1}\expandafter\ddefloop\fi}
\def\ddef#1{\expandafter\def\csname bb#1\endcsname{\ensuremath{\mathbb{#1}}}}
\ddefloop ABCDEFGHIJKLMNOPQRSTUVWXYZ\ddefloop
\def\ddefloop#1{\ifx\ddefloop#1\else\ddef{#1}\expandafter\ddefloop\fi}
\def\ddef#1{\expandafter\def\csname b#1\endcsname{\ensuremath{\mathbf{#1}}}}
\ddefloop ABCDEFGHIJKLMNOPQRSTUVWXYZ\ddefloop
\def\ddef#1{\expandafter\def\csname c#1\endcsname{\ensuremath{\mathcal{#1}}}}
\ddefloop ABCDEFGHIJKLMNOPQRSTUVWXYZ\ddefloop
\def\ddef#1{\expandafter\def\csname h#1\endcsname{\ensuremath{\widehat{#1}}}}
\ddefloop ABCDEFGHIJKLMNOPQRSTUVWXYZabcdefghijklmnopqrsuvwxyz\ddefloop    
\def\ddef#1{\expandafter\def\csname hc#1\endcsname{\ensuremath{\widehat{\mathcal{#1}}}}}
\ddefloop ABCDEFGHIJKLMNOPQRSTUVWXYZ\ddefloop
\def\ddef#1{\expandafter\def\csname t#1\endcsname{\ensuremath{\widetilde{#1}}}}
\ddefloop ABCDEFGHIJKLMNOPQRSTUVWXYZ\ddefloop
\def\ddef#1{\expandafter\def\csname tc#1\endcsname{\ensuremath{\widetilde{\mathcal{#1}}}}}
\ddefloop ABCDEFGHIJKLMNOPQRSTUVWXYZ\ddefloop

\usepackage{tikz}

\newtheorem{assumption}{Assumption}
\newtheorem{fact}{Fact}
\newtheorem{theorem*}{Theorem}

\usepackage{prettyref}
\newcommand{\pref}[1]{\prettyref{#1}}

\newcommand{\savehyperref}[2]{\texorpdfstring{\hyperref[#1]{#2}}{#2}}
\newrefformat{eq}{\savehyperref{#1}{\textup{(\ref*{#1})}}}
\newrefformat{eqn}{\savehyperref{#1}{Equation~\ref*{#1}}}
\newrefformat{fact}{\savehyperref{#1}{Fact~\ref*{#1}}}
\newrefformat{con}{\savehyperref{#1}{Conjecture~\ref*{#1}}}
\newrefformat{lem}{\savehyperref{#1}{Lemma~\ref*{#1}}}
\newrefformat{def}{\savehyperref{#1}{Definition~\ref*{#1}}}
\newrefformat{line}{\savehyperref{#1}{line~\ref*{#1}}}
\newrefformat{thm}{\savehyperref{#1}{Theorem~\ref*{#1}}}
\newrefformat{corr}{\savehyperref{#1}{Corollary~\ref*{#1}}}
\newrefformat{sec}{\savehyperref{#1}{Section~\ref*{#1}}}
\newrefformat{app}{\savehyperref{#1}{Appendix~\ref*{#1}}}
\newrefformat{ass}{\savehyperref{#1}{Assumption~\ref*{#1}}}
\newrefformat{ex}{\savehyperref{#1}{Example~\ref*{#1}}}
\newrefformat{fig}{\savehyperref{#1}{Figure~\ref*{#1}}}
\newrefformat{tab}{\savehyperref{#1}{Table~\ref*{#1}}}
\newrefformat{alg}{\savehyperref{#1}{Algorithm~\ref*{#1}}}
\newrefformat{rem}{\savehyperref{#1}{Remark~\ref*{#1}}}
\newrefformat{conj}{\savehyperref{#1}{Conjecture~\ref*{#1}}}
\newrefformat{prop}{\savehyperref{#1}{Proposition~\ref*{#1}}}
\newrefformat{proto}{\savehyperref{#1}{Protocol~\ref*{#1}}}
\newrefformat{prob}{\savehyperref{#1}{Problem~\ref*{#1}}}
\newrefformat{claim}{\savehyperref{#1}{Claim~\ref*{#1}}}

\newcommand{\VC}{\mathrm{VC}}

\newcommand{\lossclass}{\ell^{01}}

\newcommand{\unif}{\mathrm{Unif}}

\newcommand{\feprm}{\wh{f}_\mathrm{EPRM}}
\newcommand{\dis}{\mathrm{dis}}

\newcommand{\rad}{\mathfrak{R}}

\newcommand{\ezsq}{\textsf{EZ.Sq}}
\newcommand{\pmsq}{\textsf{PM.Sq}}
\newcommand{\ezlog}{\textsf{EZ.Log}}
\newcommand{\pmlog}{\textsf{PM.Log}}
\newcommand{\dsq}{\textsf{DebiasedSq}}

%% file: intro.tex
\section{Introduction}
We study Learning from Label Proportions (LLP), which is a framework for weakly supervised learning. In the standard supervised learning framework, the learner has access to a dataset of $n$ i.i.d.~labeled examples $\crl{(x_i,y_i)}_{i=1}^n \in (\cX\times \cY)^n$, and the goal is to learn an accurate predictor $\wh{f}: \cX \mapsto \cY$. In the LLP framework, the learner does not get access to the true labels $\crl{y_i}_{i=1}^n$; instead, training data is organized into ``bags'' which contain multiple unlabeled examples, and only the average (or ``aggregated'') label within the bag is provided to the learner.

The LLP framework has been studied in a long line of work, dating back to \cite{kuck2005learning, musicant2007supervised}. LLP is motivated by practical machine learning problems where individual labels are expensive to obtain or unavailable, see, e.g., applications in high energy physics \citep{dery2017weakly}, election prediction \citep{sun2017probabilistic}, and RADAR image classification \citep{ding2017learning}. More recently, LLP was proposed as a mechanism to provide user privacy \citep{diemert2022lessons}; for example, in the Apple SKAN API~\citep{SKAdNetwork} and Google Chrome's Private Aggregation API~\citep{chrome}, only aggregated labels are provided for ad conversion reporting.



\paragraph{Problem Formulation.}
Let $\cX$ represent the instance space and $\cY = \{0,1\}$ denote the label space. In LLP, we are given $n$ bags of examples $\crl{(B_i, \alpha_i)}_{i=1}^n
$ where each bag $B_i = \crl{x_{i,j}}_{j=1}^k$ contains $k$ instances and $\alpha_i = \frac{1}{k}\sum_{j=1}^k y_{i,j}$ is the average label within the bag. We assume that each instance $(x_{i,j}, y_{i,j}) \sim \cD$ is independently and identically distributed (i.i.d.) according to an unknown distribution $\cD$. We also use $(B,\alpha) \sim \cD$ to indicate that a bag is drawn from $\cD$. Unlike supervised learning, in LLP the learner does not get individual labels $y_{i,j}$, but only the aggregated label $\alpha_i$ for a collection of $k$ random instances.

We study the PAC learning objective of finding an accurate instance-level predictor from label proportion data that competes with the best predictor in specified function class $\cF \subseteq \cY^\cX$. Given parameters $\eps, \delta \in (0,1)$, we seek a predictor $\wh{f}: \cX \to \cY$ that satisfies the following condition with probability at least $1-\delta$:
\begin{align*}
  \cL(\wh{f}) \le \inf_{f \in \cF} \cL(f) + \eps, \quad \text{where} \quad \cL(f) \coloneqq \En_{(x,y) \sim \cD}[ \lossclass(y, f(x))], \numberthis\label{eq:learning-obj}
\end{align*}
and $\ell^{01}: \cY \times \cY \to \{0,1\}$ is the classification loss, defined as $\ell^{01}(y, \wh{y}) = \ind{y \ne \wh{y} }$. When the bag size $k=1$, this becomes the classic binary classification setup, for which it is known precisely that the VC dimension of $\cF$ characterizes the sample complexity of learning.

The fundamental question in LLP is to establish the sample complexity, i.e., the number of bags $n$ required to guarantee Eq.~\eqref{eq:learning-obj}, in terms of the VC dimension of $\cF$, bag size $k$, and accuracy parameters $\eps, \delta$. Our paper investigates several proposed learning rules designed to directly minimize classification loss. We establish generalization bounds for these learning rules under both the realizable setting (where $\inf_{f \in \cF} \cL(f) = 0$ and fast $1/n$ rates are possible) and the agnostic setting (where $\inf_{f \in \cF} \cL(f)$ can be arbitrary, and one gets slow $1/\sqrt{n}$ rates). Specifically, we adopt the optimistic rates framework \citep{panchenko2002some, srebro2010optimistic} which uses localized uniform convergence bounds to show generalization guarantees that interpolate between the realizable and agnostic setting.

\paragraph{Notation.} We denote the marginal label proportion $p = \Pr[y=1]$. Whenever the function class $\cF$ is clear from the context, we denote $d = \mathrm{VC}(\cF)$. We adopt standard big-oh notation and use $\wt{O}(\cdot)$ to hide $\mathrm{poly}(\log k, \log n)$ dependencies in our bounds.

\subsection{Our Contributions}
We obtain the following results on LLP for the classification objective \eqref{eq:learning-obj}.

\paragraph{Success and Failure of Empirical Proportional Risk Minimization (\pref{sec:eprm}):} We study the Empirical Proportional Risk Minimization (EPRM) learning rule \citep{yu2014learning}, which is a natural extension of Empirical Risk Minimization (ERM) to the LLP setting. Concretely we prove that under realizability, $\feprm \coloneqq \argmin_{f \in \cF} \tfrac{1}{n} \sum_{i=1}^n \indd{\tfrac{1}{k} \sum_{j=1}^k f(x_{i,j}) \ne \alpha_t}$ achieves the sample complexity guarantee
\begin{align*}
    n = O \prn*{ \frac{d\log k \cdot \log (1/\eps) + \log(1/\delta)}{\eps} }.
\end{align*}
However, in the agnostic setting we show that EPRM cannot attain polynomial (in $k$) sample complexity, and similar ``folklore'' learning rules based on minimization of proportional square or log losses even fail to return a predictor with constant suboptimality.




\paragraph{Optimistic Rates for Debiased Square Loss (\pref{sec:debiased-square-loss}):}
To address the failure of proportional loss learning rules in the agnostic setting, we consider a simple debiased variant of the proportional square loss. We show that the debiased square loss learning rule $\wh{f}_\mathrm{DSQ}$ achieves the optimistic rate:
\begin{align*}
    \cL(\wh{f}_\mathrm{DSQ}) \le L^\star + \wt{O} \prn*{\frac{k^2 \prn*{d +\log(1/\delta) }}{n}  + \sqrt{\frac{ L^\star \cdot k^2\prn*{d + \log(1/\delta) }}{n}}
    },
\end{align*} 
where we denote $L^\star = \min_{f \in \cF} \cL(f)$. Here, observe that under realizability (with $L^\star = 0$), $\wh{f}_\mathrm{DSQ}$ enjoys a fast $1/n$ rate, while in the agnostic setting we recover the $1/\sqrt{n}$ rate, both of which are optimal (up to log factors) in terms of $d$, $\log(1/\delta)$, and $n$.

%

\paragraph{Optimistic Rates for EasyLLP (\pref{sec:easyllp-main-text}):}
We study an alternative approach called EasyLLP, which was recently proposed by \cite{busa2023easy}. Specialized to the classification setting, they show a $1/\sqrt{n}$ rate (hiding dependence on $d$, $k$, and $\delta$). We improve upon their result to show an optimistic rate similar to the one achieved by the debiased square loss, showing that EasyLLP can indeed adapt to realizability. Our analysis reveals a curious phenomenon: in the realizable setting, EasyLLP exhibits a separation between loss estimation (which is necessarily $\Omega(1/\sqrt{n})$ even for the optimal predictor $f^\star$) and learning (which is $\wt{O}(1/n)$ by the optimistic rates guarantee).

\paragraph{Lower Bounds (\pref{sec:main-lower-bounds}):} We investigate the optimal dependence on the bag size $k$, since our bounds are tight (up to log factors) in the other parameters. A trivial lower bound of $n = \Omega(1/k)$ follows because LLP with $n$ bags is only harder than supervised learning with $nk$ bags. It turns out this cannot be improved in general for all $\cF$, but we give an explicit example of a function class $\cF$ for which the minimax sample complexity has larger dependence on $k$.

\arxiv{\paragraph{Experiments (\pref{sec:experiments}):}}
\colt{\paragraph{Experiments (\pref{app:experiments}):}} We empirically evaluate gradient-based versions of the learning rules considered herein on binary classification versions of MNIST and CIFAR10 studied in \cite{busa2023easy} for a wide range of NN architectures. We find that proportion matching and the debiased square loss perform the best; furthermore, we demonstrate that the debiased square loss enjoys faster optimization than proportion matching, as early in training the debiased square loss is a better estimate of the true instance-level loss.

\subsection{Related Work}

\paragraph{Learning from Label Proportions.}
The problem of LLP has been studied in a long line of work \citep{chen2006learning,musicant2007supervised,kuck2005learning,quadrianto2008estimating}. Most works either assume some kind of distributional assumptions on bag/label generation \citep{kuck2005learning, quadrianto2008estimating, patrini2014almost, scott2020learning, zhang2022learning}, construct bags adaptively~\citep{chen2023learning, javanmard2024priorboost}, or study approaches to minimize a surrogate loss \citep{rueping2010svm, yu2013proptosvm, qi2016learning, shi2019learning, dulac2019deep, javanmard2024learning}. On the computational side, \cite{saket2021learnability, saket2022algorithms} shows that even in the realizable setting learning linear thresholds for LLP is NP-hard and study SDP relaxations for this task. Thus, the aforementioned papers are not directly relevant to the goals of this work in providing distribution-free statistical guarantees on classification loss.

Several works provide guarantees on the instance-level classification loss. \cite{yu2014learning} introduce the EPRM learning rule. While their main focus is proving guarantees on the proportional risk, they show how to translate these to instance-level guarantees when the bags are ``pure''---meaning the label proportions $\alpha_t$ are close to $0$ or $1$. They also give numerical bounds which indicate that EPRM achieves instance-level guarantees under realizability. \cite{chen2023learning} introduce a similar debiased square loss learning rule and prove a $O(1/\sqrt{n})$ rate in the agnostic setting; while \cite{busa2023easy} introduce the EasyLLP framework and prove a $O(1/\sqrt{n})$ rate.



\paragraph{Optimistic Rates.}
Optimistic rates date back to seminal work of \cite{vapnik2015uniform} and were expanded upon by \cite{bousquet2002concentration, koltchinskii2000rademacher, bartlett2005local} using the technique of localized Rademacher complexities. Optimistic rates have been studied in various other contexts such as optimization with smooth losses \citep{srebro2010optimistic}, multi-task learning \citep{yousefi2018local,watkins2023optimistic}, vector-valued learning \citep{reeve2020optimistic}, and overparameterized regression \citep{zhou2021optimistic, zhou2022non,  zhou2023uniform}.

%% file: proportion_matching.tex
\section{Empirical Proportional Risk Minimization}\label{sec:eprm}

The most direct approach to finding a good predictor is the so-called \emph{empirical proportional risk minimization} (EPRM) approach, which simply returns a predictor that matches the most label proportions \citep{yu2014learning}. Concretely, we can consider the learning rule
\begin{align*}
\feprm \coloneqq \argmin_{f \in \cF} \frac{1}{n} \sum_{i=1}^n \ind{\frac{1}{k} \sum_{j=1}^k f(x_{i,j}) \ne \alpha_i}.\numberthis \label{eq:eprm}
\end{align*}
Observe that in the setting of $k=1$, EPRM recovers the classical empirical risk minimizer (ERM). 

More generally, one can also define learning rules which minimize some other bag loss $\ell: [0,1]\times [0,1] \to \bbR$ between the predicted label proportions $\wh{\alpha} = \frac{1}{k}\sum_i f(x_i)$ and the true label proportions $\alpha$, i.e., the square loss $\ell_\mathrm{SQ}(\alpha,\wh{\alpha}) = (\alpha -\wh{\alpha})^2$ or the log loss $\ell_\mathrm{LOG}(\alpha,\wh{\alpha}) = -\alpha \log \wh{\alpha} - (1-\alpha) \log (1- \wh{\alpha})$. In the literature, these are also called EPRM or proportion matching learning rules, and they are a ``folklore'' approach which, in conjunction with gradient-based methods, are competitive in practice \citep{busa2023easy}. To disambiguate, we exlusively refer to the learning rule \eqref{eq:eprm} as EPRM and call the more general class of these learning rules as proportion matching.

In this section, we provide theoretical results which substantiate conventional wisdom surrounding EPRM. First, we show that under realizability, the EPRM learning rule attains fast rates for classification. However, in the agnostic setting, we illustrate that proportion matching can be ill-behaved, as we demonstrate an example for which minimizing the proportional risk gives no guarantees on the instance-level classification performance.

\subsection{Fast Rates for EPRM Under Realizability}
We show the following generalization guaranteee for EPRM under realizability.

\begin{theorem}\label{thm:prop-matching-realizable}
Let $\cF$ be a symmetric function class, i.e. if $f \in \cF$, then $1-f \in \cF$. Let bag size $k \ge 11$, $\eps \in \prn*{0, 1/(4k)}$, and $\delta \in (0,1)$. As long as $n = O \prn{ \tfrac{d\log k \cdot \log (1/\eps) + \log(1/\delta)}{\eps} }$, for any realizable distribution $\cD$, with probability at least $1-\delta$ over the draw of the sample, $\cL(\feprm) \le \eps$.
\end{theorem}

Previous works have suggested that EPRM (or more generally, proportion matching learning rules) can succeed under realizability: \cite{yu2014learning} give numerical evidence to show that a guarantee on bag proportions can be translated to a guarantee on the instance-level classification error, and \cite{busa2023easy} show that under some conditions on the loss function, minimizers for a population proportion matching loss are also minimizers for the instance-level loss. However, to the best of our knowledge, \pref{thm:prop-matching-realizable} is the first result which provides a concrete generalization bound for bounded VC classes for the EPRM learning rule.

A few comments about \pref{thm:prop-matching-realizable} are in order.
\begin{itemize}
    \item The assumption that $\cF$ is symmetric is mild and due to technical reasons. Note that any nonsymmetric $\cF$ can be enlarged to a symmetric one with VC dimension at most $2d+1$.
    \item The bag size assumption is technically required, and we conjecture it can be removed. It is a mild assumption since if one has bags of size $k \le 10$, then one can preprocess the dataset to combine sets of bags to form a larger bag of size at least 11, and then compute $\feprm$ on the preprocessed dataset. This achieves the same guarantee (albeit with smaller range of $\eps$).
\end{itemize}

\begin{proof}[Proof of \pref{thm:prop-matching-realizable}.]
The proof has three steps. First, we show via standard uniform convergence arguments that the predictor $\feprm$ must have small population proportional matching error. Next, we relate the proportion matching error to the classification error to show that $\feprm$ must have classification error $\cL(\feprm) \notin [\eps, 1-\eps]$. Finally, we show that $\feprm$ must have classification error $\cL(\feprm) \le \eps$, as otherwise we would have selected the predictor $1-\feprm$.

\paragraph{Step 1.} We rewrite our problem as a binary classification problem and apply standard uniform convergence guarantees.
Define the function class $\cG: \cX^n \times [0,1] \to \crl{0,1}$ as
\begin{align*}
    \cG = \crl*{ g_f: (B, \alpha) \mapsto \ind{\frac{1}{k}\sum_{j=1}^k f(x_j) \ne \alpha} : f \in \cF}.
\end{align*}
We claim that $\VC(\cG) \le O (d \log k)$. To show this, consider any $X = \crl{(B_1, \alpha_1), \cdots, (B_m, \alpha_m)}$, and define the projection of $\cG$ onto $X$ as $\cG_X \coloneqq \crl{ \prn{g_f(B_1, \alpha_1), \cdots, g_f(B_m, \alpha_m)}: f \in \cF}$. It suffices to show that when $m = O(d \log k)$, the set of labellings for $X$ is of size $\abs{\cG_X} < 2^m$. Observe that the labelling of $X$ is determined by the labellings of $\cF$ on the $mk$ points, so by Sauer's lemma $\abs{\cG_X} \le (emk/d)^d$. Therefore, when $m = O(d \log k)$ we have $\abs{\cG_X} < 2^m$.

The EPRM can be written as $\feprm \coloneqq \argmin_{f \in \cF} \frac{1}{n} \sum_{i=1}^n \ell^{01}(0, g_f(B_i, \alpha_i))$.
Applying the uniform convergence guarantee for VC classes \citep[e.g.,][]{shalev2014understanding}, we see that as long as $n = O \prn{\frac{d\log k \cdot \log (1/\eps) + \log(1/\delta)}{\eps}}$, with probability at least $1-\delta$ we have
\begin{align*}
    \Pr_{(B, \alpha) \sim \cD} \bigg[ \frac{1}{k} \sum_{j=1}^k \feprm(x_j) \ne \alpha \bigg] \le \eps. \numberthis\label{eq:ub-prop-loss}
\end{align*}
Henceforth, we will condition on the event in Eq.~\eqref{eq:ub-prop-loss} holding.

\paragraph{Step 2.} Now we show that a proportional risk guarantee of the form Eq.~\eqref{eq:ub-prop-loss} translates to a guarantee on the instance-level loss $\cL(\cdot)$. Let $f^\star \in \cF$ be the optimal predictor that achieves $\cL(f^\star) = 0$, and for any $f \in \cF$ define $\dis(f, f^\star) \coloneqq \crl*{x \in \cX: f(x) \ne f^\star(x)}$ to be the disagreement region on which $f$ and $f^\star$ disagree. By definition, $\cL(\feprm) = \Pr_{x \sim \cD}[x \in \dis(\feprm, f^\star)]$. We will show that $\cL(\feprm) \notin [\eps, 1-\eps]$.

To do so, we bound the probability that $\feprm$ does not match the proportion on a freshly sampled bag. We already have an upper bound on this from Eq.~\eqref{eq:ub-prop-loss}. Now we compute a lower bound.
\begin{align*}
     \Pr_{(B,\alpha) \sim \cD}\bigg[ \frac{1}{k} \sum_{j=1}^k \feprm(x_j) \ne \alpha \bigg] &\ge \Pr_{(B,\alpha) \sim \cD} \bigg[ \sum_{j=1}^k \ind{x_j \in \dis(\feprm, f^\star)} \text{ is odd} \bigg] \\
     &= \frac{1}{2} - \frac{1}{2} \prn*{1- 2\cL(\feprm)}^k.\numberthis\label{eq:lb-prop-loss}
\end{align*}
The first inequality follows because if the bag contains an odd number of points in the disagreement set, then it is impossible for $\feprm$ and $f^\star$ to have the same proportional label.

For sake of contradiction, suppose that $\cL(\feprm) \in [\eps, 1-\eps]$. Then we know $\tfrac{1}{2} - \tfrac{1}{2} \prn{1- 2\cL(\feprm)}^k \ge \tfrac{1}{2} - \tfrac{1}{2} \prn{1- 2\eps }^k$.
However, if $\eps \in (0,1/2)$, we arrive at a contradiction, since $\eps \in (0,1/2)$ implies that for any bag size $k \ge 2$, we have $1-2\eps > (1-2\eps)^k$, which implies that $\eps < \frac{1}{2} - \frac{1}{2} \prn*{1- 2\eps }^k$, so Eqs.~\eqref{eq:ub-prop-loss} and \eqref{eq:lb-prop-loss} cannot simultaneously hold. Therefore we must have $\cL(\feprm) \notin [\eps, 1-\eps]$.

\paragraph{Step 3.} Now we establish that $\cL(\feprm) \notin [1-\eps, 1]$. We claim the following: for any near optimal predictor $\wt{f} \in 
\cF_\eps \coloneqq \crl{f \in \cF: \cL(f) \le \eps}$, we must have 
\begin{align*}
    \frac{1}{n} \sum_{i=1}^n \ind{ \frac{1}{k}\sum_{j=1}^k \wt{f}(x_{i,j}) \ne \alpha_i} < \frac{1}{n} \sum_{i=1}^n \ind{ \frac{1}{k}\sum_{j=1}^k 1 - \wt{f}(x_{i,j}) \ne \alpha_i}. \numberthis\label{eq:epr-dominance}
\end{align*}
Call this event $\cE(\wt{f})$.
From here, the result that $\cL(\feprm) \notin [1-\eps, 1]$ follows because if $\cL(\feprm) \in [1-\eps, 1]$, then it could not have been the EPRM, as the predictor $1 - \feprm$ has strictly better empirical proportional risk and also lies in the class $\cF$ by the symmetric assumption.

We now prove Eq.~\eqref{eq:epr-dominance}. Consider any predictor $\wt{f} \in \cF_\eps$. Define the indicator variable $Z_i \in \crl{0,1}$ as $Z_i = \indd{ \frac{1}{k}\sum_{j=1}^k \wt{f}(x_{i,j}) = \alpha_i} \cdot \indd{ \frac{1}{k}\sum_{j=1}^k 1 - \wt{f}(x_{i,j}) \ne \alpha_i}$.
We see that $\crl*{\frac{1}{n}\sum_i Z_i > 1/2} \subseteq \cE(\wt{f})$. We bound the expectation of $Z_i$ as:
\begin{align*}
    \En[Z_i] \ge \Pr_{(B,\alpha)\sim \cD} \brk*{ \forall x_j \in B: x_j \notin \dis(\wt{f}, f^\star) \text{ and } \alpha \ne \frac12}. 
\end{align*}
To lower bound this, we can bound the two events separately. First we have
\begin{align*}
    \Pr_{(B,\alpha)\sim \cD}\brk*{ \forall x_j \in B: x_j \notin \dis(\wt{f}, f^\star) } \ge (1-\eps)^k \ge \frac{3}{4},
\end{align*}
where the last inequality is true whenever $\eps \le 1/(4k)$. In addition,
\begin{align*}
    \Pr_{(B,\alpha)\sim \cD}[\alpha \ne 1/2] = 1 - \Pr_{(B,\alpha)\sim \cD}[\alpha = 1/2] \ge 1 - 2^{-k} {k \choose k/2} \ge 1 - \frac{1}{\sqrt{3k/2 + 1/2}},
\end{align*}
by Stirling's approximation. Using the law of total probability we get
\begin{align*}
    1 &\ge \Pr_{(B,\alpha)\sim \cD} \Big[ \forall x_j \in B: x_j \notin \dis(\wh{f}, f^\star) \text{ or } \alpha \ne \frac12\Big] \\
    &= \Pr_{(B,\alpha)\sim \cD} \Big[ \forall x_j \in B: x_j \notin \dis(\wh{f}, f^\star) \Big] + \Pr_{(B,\alpha)\sim \cD}\Big[\alpha \ne \frac12\Big] \\
    &\quad\quad\quad - \Pr_{B} \Big[\forall x_j \in B: x_j \notin \dis(\wh{f}, f^\star) \text{ and } \alpha \ne \frac12 \Big] \\
    &\ge \frac34 + 1 - \frac{1}{\sqrt{3k/2+1/2}} - \Pr_{(B,\alpha)\sim \cD} \Big[ \forall x_j \in B: x_j \notin \dis(\wh{f}, f^\star) \text{ and } \alpha \ne \frac12\Big],
\end{align*}
so therefore
\begin{align*}
    \Pr_{(B,\alpha)\sim \cD} \Big[ \forall x_j \in B: x_j \notin \dis(\wh{f}, f^\star) \text{ and } \alpha \ne \frac12\Big] \ge \frac34 - \frac{1}{\sqrt{3k/2+1/2}}.
\end{align*}
Whenever $k \ge 11$ the RHS is at least $0.507$. As a consequence by Hoeffding's inequality, we have $\Pr\brk{\cE(\wt{f})^c} \le \Pr\brk{\tfrac{1}{n}\sum_i Z_i \le 1/2} \le \exp\prn*{-2n \cdot 0.07^2}$. By union bound, we have $\Pr\brk{\exists f \in \cF_\eps: \cE(\wt{f})^c} \le \Gamma_\cF(nk) \cdot \exp\prn*{-2n \cdot 0.07^2}$, where $\Gamma_\cF: \bbN \to \bbN$ is the growth function for $\cF$.
Setting the RHS to $\delta$ and using Sauer's lemma we get that as long as $n = O\prn*{d\log k + \log(1/\delta)}$, the event $\cE(\wt{f})$ holds for all $\wt{f} \in \cF_\eps$. 

\paragraph{Putting it together.} Therefore, with probability at least $1-2\delta$, $\cL(\feprm) \le \eps$ as long as $\eps \le 1/(4k)$ and \begin{align*}
    n = O \prn*{\frac{d\log k \cdot \log (1/\eps) + \log(1/\delta)}{\eps}}.
\end{align*}
After rescaling $\delta$, this concludes the proof of \pref{thm:prop-matching-realizable}.
\end{proof}

\subsection{Proportion Matching Fails in Agnostic Setting}

In the agnostic setting, we illustrate how proportion matching can perform quite poorly.

\paragraph{Example: EPRM may require $\Omega(2^k)$ sample complexity.}
    Fix any $\eps \in (0,1/2)$ and consider the input space $\cX = \crl{x}$, with $\cD$ given by $(x,1)$ with probability $1/2 + \eps$ and $(x,0)$ with probability $1/2-\eps$. Let $\cF$ consist of two functions $f_0(x) = 0$ and $f_1(x) = 1$. The optimal predictor within the class $\cF$ is $f_1$. However, observe that the loss estimates for proportion matching are $\wh{L}(f_0) \coloneqq 1 - \frac{1}{n}\sum_{i=1}^n \indd{\alpha_i = 0}$ and $\wh{L}(f_1) \coloneqq 1 - \frac{1}{n}\sum_{i=1}^n \indd{\alpha_i = 1}$.
For small $\eps$, unless the number of bags is exponential in $k$, with constant probability we do not see any ``pure'' bags (with $\alpha_i = 0$ or $\alpha_i=1$), so we have no way of distinguishing which of $f_0$ and $f_1$ achieves smaller loss.

\paragraph{Proportion matching can fail.} One may object that the previous failure mode is due to the fact that we are using a noncontinuous measure of discrepancy between the predicted and true label proportion, and such issues can be resolved if we minimize a continuous measure of discrepancy. We show that this does not help, as proportion matching approaches can return a predictor with constant suboptimality. This is because in the agnostic setting, the predictor which matches the bag-level proportions may not be the optimal instance-level predictor.

Consider the learning rule that minimizes the proportional square loss:
\begin{align*}
    \wh{f}_\mathrm{SQ} \coloneqq \argmin_{f \in \cF} \wh{L}_\mathrm{SQ}(f) = \frac{1}{n} \sum_{i=1}^n \bigg(\frac{1}{k} \sum_{j=1}^k f(x_{i,j}) - \alpha_i\bigg)^2. 
\numberthis\label{eq:square-prop-learning-rule}
\end{align*}
We show that in general, minimizing the proportional square loss can fail in the agnostic setting.
\begin{proposition}\label{prop:prop-matching-failure}
There exists a $\cF$ with $\VC(\cF) = 1$ and distribution $\cD$ such that for any $\delta \in (0,1)$, bag size $k \ge 7$, and sample size $n = \Omega\prn{ \log (1/\delta) }$, with probability at least $1-\delta$, the learning rule $\wh{f}_\mathrm{SQ}$ is $1/3$-suboptimal.
\end{proposition}

The proportional square loss learning rule, as well as the proportional log loss learning rule
\begin{align*}
    \wh{f}_\mathrm{LOG} \coloneqq \argmin_{f\in \cF} \frac1n \sum_{i=1}^n -\alpha_i \cdot \log\bigg(\frac1k \sum_{j=1}^k f(x_{i,j}) \bigg) - (1-\alpha_i) \cdot \log \bigg( 1-\frac1k \sum_{j=1}^k f(x_{i,j}) \bigg) \numberthis \label{eq:log-prop-learningrule}
\end{align*}
are regarded as folklore learning rules, and they were evaluated in the context of gradient-based learning \citep{busa2023easy}. \citeauthor{busa2023easy} show that while gradient-based minimization of either the proportional square or log loss performs well in practice, it can fail in synthetic experimental settings. \pref{prop:prop-matching-failure} demonstrates a simple theoretical failure mode for $\wh{f}_\mathrm{SQ}$. For completeness, in \pref{app:failure-flog} we provide a similar result for the failure of $\wh{f}_\mathrm{LOG}$ on the same construction, but note that even in the standard classification setting ($k=1$) it is well known that minimizing surrogate losses like the log loss do not necessarily give
guarantees on the classification error in the agnostic setting \citep{ben2012minimizing}. Lastly, we remark that Appendix A of \cite{scott2020learning} shows an example of similar flavor that in the limit as the bag size $k \to \infty$, proportion matching learning rules can suffer constant suboptimality.

\begin{proof}[Proof of \pref{prop:prop-matching-failure}.]
Let $\cX = \crl{x^{(1)}, x^{(2)}}$. Let $\cF = \crl{f_1, f_2}$ where $f_1(x) = \indd{x = x^{(1)}}$ and $f_2(x) = \indd{x = x^{(2)}}$. The distribution $\cD$ is $(x,y) \sim \mathrm{Unif}\prn{ \crl{(x^{(1)}, 1), (x^{(1)}, 0), (x^{(2)}, 1)}}$. We can calculate that $\cL(f_1) = 2/3$ and $\cL(f_2) = 1/3.$ However we also have $\En f_1 = 2/3$ while $\En f_2 = 1/3$, and $p = 2/3$. While $f_1$ in expectation matches the marginal label proportion, it is actually $1/3$-suboptimal compared to $f_2$.

We compute the expectations of the bag-level losses for $f_1$ and $f_2$. For $f_1$ we have
\begin{align*}
    \En \bigg[\Big(\frac{1}{k} \sum_{j} f_1(x_j) - \alpha\Big)^2 \bigg] = \frac{1}{k^2} \En \bigg[\Big( \sum_{j} f_1(x_j) - y_i\Big)^2 \bigg] = \frac{1}{k^2} \En \bigg[\sum_j \prn*{f_1(x_j) - y_j}^2\bigg] = \frac{2}{3k}.
\end{align*}
For $f_2$ we have
\begin{align*}
    \En \bigg[\Big(\frac{1}{k} \sum_{j} f_2(x_j) - \alpha\Big)^2 \bigg] 
    &= \frac{1}{k^2} \En \bigg[\sum_j \prn*{f_2(x_j) - y_j}^2\bigg] + \frac{1}{k^2} \En \bigg[\sum_{j \ne j' } \prn*{f_2(x_j) - y_j} \prn*{f_2(x_{j'}) - y_{j'}}\bigg] \\
    &= \frac{1}{3k} + \frac{k-1}{9k} 
    = \frac{k+2}{9k}.
\end{align*}
Fix any $k \ge 7$. Then the expectation of the bag-level loss of $f_1$ is at most $2/21$ while the expectation of $f_2$ is at least $1/9$. By Hoeffding's inequality, we have with probability at least $1-\delta$, that $\wh{\cL}(f_1) \le 2/21 + \sqrt{2\log (1/\delta)/n}$ and $\wh{\cL}(f_2) \ge 1/9 - \sqrt{2\log (1/\delta)/n}$.
So as long as $n = \Omega\prn{ \log (1/\delta) }$, the proportional square loss learning rule will return the wrong predictor.
\end{proof}

%% file: unbiased_square.tex
\section{Debiased Square Loss}\label{sec:debiased-square-loss}
In this section, we show that a simple debiasing of the square loss $\wh{L}_\mathrm{SQ}(\cdot)$ results in a learning rule that achieves optimal rates in both the realizable and the agnostic settings. Computing the expectation of the square loss, for any predictor $f$,
\begin{align*}
\En \brk*{\wh{L}_\mathrm{SQ}(f)} 
    &= \frac{1}{k^2} \cdot \En_{(B,\alpha)\sim \cD} \bigg[ \sum_{j} (f(x_j) - y_j)^2 + \sum_{j\ne j'} (f(x_j) - y_j)(f(x_{j'}) - y_{j'}) \bigg]\\
    &\overset{(i)}{=} \frac{1}{k} \cdot \cL(f) + \frac{k-1}{k} \cdot \En_{(x, y), (x', y')\sim \cD} \Big[(f(x) - y)(f(x') - y') \Big] \\
    &\overset{(ii)}{=} \frac{1}{k} \cdot \cL(f) + \frac{k-1}{k} \cdot \prn*{\En f - p}^2. \numberthis\label{eq:squareloss-expectation}
\end{align*}
Equality $(i)$ uses the fact that for any $j$, $(f(x_j) - y_j)^2 = \ind{f(x_j) \ne y_j}$. Equality $(ii)$ uses independence of the samples. 

Rearranging Eq.~\eqref{eq:squareloss-expectation}, we can see that the quantity $k\cdot \wh{L}_\mathrm{SQ}(f) - (k-1)\cdot \prn*{\En f - p}^2$ is an unbiased estimate of $\cL(f)$. Of course, the caveat is that we do not have access to the values $\En f$ and $p$, but we can replace them with their empirical counterparts, giving us the debiased square loss learning rule:
\begin{align*}
    \wh{f}_{\mathrm{DSQ}} \coloneqq \argmin_{f \in \cF} ~ \wh{L}_{\mathrm{DSQ}}(f) = \frac{1}{n} \sum_{i=1}^n k \cdot \bigg( \frac{1}{k} \sum_{j=1}^k f(x_{i,j}) - \alpha_i \bigg)^2 - (k-1) \prn*{\wh{\En} f - \wh{p}}^2,
\end{align*}
where $\wh \En f = \frac{1}{nk}\sum_{i=1}^n \sum_{j=1}^k f(x_{i,j})$ and $\wh{p} = \frac{1}{n} \sum_{i=1}^n \alpha_i$ are empirical estimates of $\En f$ and $p$ respectively. Strictly speaking, $\wh{L}_\mathrm{DSQ}(f)$ is not an unbiased estimate of $\cL(f)$, but if $n$ is large enough, the second term approximates $\prn{\En f - p}^2$ closely.

A similar debiasing idea for the proportional square loss was proposed by \cite{chen2023learning}, where in their Theorem 2 they prove the $1/\sqrt{n}$ rate in the agnostic setting; due to difference in the bag generation assumption (they consider generating bags by resampling a dataset without replacement), our debiasing term takes a slightly different form.

\subsection{Main Result: Optimistic Rates for Debiased Square Loss}\label{sec:optimistic-rates-dsl}
\begin{theorem}[Sample Complexity Bound for $\wh{f}_\mathrm{DSQ}$]\label{thm:lossbagsq-opt-rate}
Let $\delta \in (0,1)$. Fix any distribution $\cD$ and any function class $\cF$, and let $L^\star = \inf_{f \in \cF} \cL(f)$. With probability at least $1-\delta$, we have
\begin{align*}
    \cL(\wh{f}_\mathrm{DSQ}) \le L^\star + \wt{O} \prn*{\frac{k^2 \prn*{d +\log(1/\delta) }}{n}  + \sqrt{\frac{ L^\star \cdot k^2\prn*{d + \log(1/\delta)}}{n}}
    }.
\end{align*} 
\end{theorem}

\pref{thm:lossbagsq-opt-rate} shows that under realizability (with $L^\star = 0$), $\wh{f}_\mathrm{DSQ}$ enjoys a fast $1/n$ rate, while in the agnostic setting it achieves the $1/\sqrt{n}$ rate, both of which are optimal (up to log factors) in terms of $d$, $\log(1/\delta)$, and $n$. The dependence on $k$ is certainly loose, as under realizability, $\wh{f}_\mathrm{DSQ}$ and $\wh{f}_\mathrm{EPRM}$ are identical learning rules, and \pref{thm:prop-matching-realizable} only has a $\log k$ dependence. We leave sharpening the dependence on $k$ to future work.

The proof of \pref{thm:lossbagsq-opt-rate} can be found in \pref{app:proof-debiased-square}. At a high level, we separately show uniform convergence bounds for both the square loss and the bias correction term to their expectations, and then we combine the guarantees to get a optimistic rate bound for $\wh{L}_\mathrm{DSQ}(\cdot)$, which is an unbiased estimate of $\cL(f)$, giving us the final guarantee.

\paragraph{Revisiting what happens if we minimize $\wh{L}_\mathrm{SQ}$.} Our analysis provides an answer to the question raised by \cite{busa2023easy} on understanding when proportion matching is consistent. Suppose the learner minimizes the square loss of Eq.~\eqref{eq:square-prop-learning-rule}, $\wh{L}_\mathrm{SQ}(f) = \frac{1}{n} \sum_{i=1}^n k \cdot \prn{ \frac{1}{k} \sum_{j=1}^k f(x_{i,j}) - \alpha_i }^2$. (For sake of comparison, we multiply the loss by a factor of $k$.) One can show an optimistic-rate style guarantee for $\wh{f}_\mathrm{SQ}$, albeit one that is weaker than \pref{thm:lossbagsq-opt-rate}. Let $L_\mathrm{SQ}(\cdot)$ denote the expectation of $\wh{L}_\mathrm{SQ}$. In the proof of \pref{thm:lossbagsq-opt-rate}, we get the following guarantee on $\wh{f}_\mathrm{SQ}$:
\begin{align*}
    L_\mathrm{SQ}(\wh{f}_\mathrm{SQ}) \le \inf_{f \in \cF} \bigg\{ L_\mathrm{SQ}(f) + \wt{O}\bigg( \frac{k d + k^2 \log(1/\delta) }{n}   + \sqrt{
    \frac{ L_\mathrm{SQ}(f) \cdot k(d + \log(1/\delta))}{n} } \bigg) \bigg\} .
\end{align*}
This bound essentially replaces the instance-level classification error $\cL(f)$ with the (larger) $L_\mathrm{SQ}(f)$ in \pref{thm:lossbagsq-opt-rate}. Applying the substitution $L_\mathrm{SQ}(f) = \cL(f) + B(f)$, where $B(f) \coloneqq (k-1) \prn*{\En f - p}^2$, we see that this bound implies
\begin{align*}
    \cL(\wh{f}_\mathrm{SQ}) \le \inf_{f \in \cF} \bigg\{ \cL(f) + B(f) + \wt{O}\bigg(  \tfrac{k d + k^2 \log(1/\delta)}{n}   + \sqrt{
    \tfrac{ \prn*{\cL(f) +B(f)} \cdot k(d + \log(1/\delta))}{n} } \bigg) \bigg\}.\numberthis\label{eq:sq-loss-guarantee}
\end{align*}
When is Eq.~\eqref{eq:sq-loss-guarantee} a useful bound? Under realizability, we have $\cL(f^\star) = 0$ and $B(f^\star) = 0$, and since any minimizer $\wh{f}_\mathrm{EPRM}$ is also a minimizer of $\wh{L}_\mathrm{SQ}(\cdot)$ and vice versa, Eq.~\eqref{eq:sq-loss-guarantee} recovers the guarantee of \pref{thm:prop-matching-realizable}, albeit with worse dependence on $k$. More generally, in the agnostic setting, as long as it is possible to achieve a good trade-off between $\cL(f)$ and $B(f)$, we expect $\wh{f}_\mathrm{SQ}$ to generalize well. This reasoning suggests why proportion matching learning rules perform well in practice \citep{busa2023easy}, despite the negative result of \pref{prop:prop-matching-failure}. In the modern over-parameterized regime, where the function class at hand (i.e., neural networks) are nearly realizable for the data distribution, Eq.~\eqref{eq:sq-loss-guarantee} delivers strong guarantees.

Lastly, observe that there is no contradiction between the guarantee of Eq.~\eqref{eq:sq-loss-guarantee} and the lower bound in \pref{prop:prop-matching-failure}: in \pref{prop:prop-matching-failure}, there is no predictor for which the two terms $\cL(f)$ and $B(f)$ are both small, as the suboptimal predictor has larger classification loss but satisfies $B(f) = 0$.

%% file: easyllp.tex
\section{EasyLLP Learning Rule}\label{sec:easyllp-main-text}
Recently, \cite{busa2023easy} proposed EasyLLP, a general recipe for constructing unbiased estimators of any instance-level loss function $\ell_\mathrm{ins}: \cY \times \bbR \to \bbR$. Given predictor $f \in \cF$ and bag $(B, \alpha)$, the EasyLLP loss estimate of $\ell_\mathrm{ins}$ can be written as
\begin{align*}
    \ell_\mathrm{EZ}\prn*{f, (B,\alpha)} &\coloneqq \frac{1}{k} \sum_{j=1}^k \prn*{k(\alpha - p) + p} \ell_\mathrm{ins} \prn*{1, f(x_j)} + \prn*{k(p - \alpha) + (1-p)} \ell_\mathrm{ins} \prn*{0, f(x_j)} \numberthis \label{eq:easyLLP-loss}
\end{align*}
Proposition 4.2 of \cite{busa2023easy} shows that for any loss $\ell_\mathrm{ins}$, $\ell_\mathrm{EZ}\prn*{f, (B,\alpha)}$ is an unbiased estimate of the population loss, e.g., $\En_{(B, \alpha)\sim \cD} \brk{\ell_\mathrm{EZ}\prn{f, (B,\alpha)}} = \En_{(x,y)\sim \cD} \brk{\ell_\mathrm{ins}(y, f(x))}$.
In this work, we consider EasyLLP instantiated with the classification loss $\ell_\mathrm{ins} = \ell^{01}$. The EasyLLP estimate takes the following form:
\begin{align*}
    \ell_\mathrm{EZ}\prn*{f, (B,\alpha)} &= \prn*{k(\alpha - p) + p} \big(1- \frac{1}{k} \sum_{j=1}^k f(x_j)\big) + \prn*{k(p - \alpha) + (1-p)} \big(\frac{1}{k} \sum_{j=1}^k f(x_j)\big). \numberthis\label{eq:easyllp-loss-est}
\end{align*}
The EasyLLP learning rule $\wh{f}_\mathrm{EZ} \coloneqq \argmin_{f \in \cF} \wh{L}_\mathrm{EZ}(f) = \frac1n \sum_{i=1}^n \ell_\mathrm{EZ}(f, (B_i,\alpha_i))$ attains the following guarantee \citep[Theorem 5.2 of][]{busa2023easy}: with probability at least $1-\delta$,
\begin{align*}
    \cL(\wh{f}_\mathrm{EZ}) \le \inf_{f \in \cF} \cL(f) + \wt{O} \bigg( \sqrt{\frac{ d+k \log(1/\delta)}{n} } \bigg).\numberthis\label{eq:easyllp-slow}
\end{align*}
However, the question of whether EasyLLP can adapt to realizability to achieve fast rates remained open. The standard observation which enables fast rates under realizability is that loss estimates for (nearly) optimal predictors converge to their expectations at a rate of $O(1/n)$. As a consequence, in supervised learning, one can show that any predictor $\wh{f}$ which achieves training error 0 (i.e., any ERM) has generalization error at most $\cL(\wh{f}) \le \wt{O}(d/n)$.

In contrast, the EasyLLP loss estimate does not satisfy this property. Even for the optimal predictor $f^\star$, the EasyLLP loss estimate of Eq.~\eqref{eq:easyllp-loss-est} is a \emph{random} quantity, since
\begin{align*}
    \ell_\mathrm{EZ}\prn*{f^\star, (B,\alpha)} &= \prn*{k(\alpha - p) + p} \cdot \prn*{1- \alpha} + \prn*{k(p - \alpha) + (1-p)} \cdot \alpha
\end{align*}
takes values depending on the bag label proportion $\alpha$ which itself is distributed as $1/k \cdot \mathrm{Bin}(k, p)$. As the following proposition shows, the loss estimate only concentrates at a $\Theta(1/\sqrt{n})$ rate.

\begin{proposition}\label{prop:easyllp-slow-est}
There exists a realizable distribution $\cD$ such that for any $\eps \in (0,1)$, the EasyLLP loss estimate of $f^\star$ with bag size $k=2$ requires $\Theta(1/\eps^2)$ samples in order for $\wh{L}_\mathrm{EZ}(f^\star) \le \eps$ with constant probability.
\end{proposition}

\begin{proof}[Proof.] Consider the setting $\cX = \crl{x_0, x_1}$ with $\cD$ that returns $(x_0, 0)$ with probability $1/2$ and $(x_1, 1)$ with probability $1/2$. Let $f^\star(x) = \ind{x = x_1}$ be the optimal predictor achieving $\cL(f^\star) = 0$. For any bag $(B, \alpha)$ the EasyLLP estimate of the loss can be written as 
\begin{align*}
    \ell_\mathrm{EZ}(f^\star,(B, \alpha)) &= \prn*{k(\alpha - 1/2) + 1/2} \cdot (1-\alpha) + \prn*{k(1/2 - \alpha) + 1/2} \cdot \alpha \\
    &= -2k\alpha^2 + 2\alpha k - k/2 + 1/2.
\end{align*}
For bag size $k=2$, we have $\ell_\mathrm{EZ}(f^\star,(B, \alpha)) \sim 1/2 \cdot \mathrm{Rad}(1/2)$.
From here, we can apply standard anti-concentration bounds for sums of Rademacher random variables. We let $x_j = \ell_\mathrm{EZ}(f^\star,(B_i, \alpha_i))$, and we can observe that by Paley-Zygmund (\pref{thm:paley-zygmund}) that
\begin{align*}
    \Pr \bigg[ \frac{1}{n} \sum_{i=1}^n x_j \ge \frac{0.1}{\sqrt{n}} \bigg] = \frac{1}{2} \Pr \bigg[ \Big| \frac{1}{n} \sum_{i=1}^n x_j \Big| \ge \frac{0.1}{\sqrt{n}} \bigg] \ge \Omega(1).
\end{align*}
Thus, if $n = c/\eps^2$ for sufficiently small $c > 0$ then with constant probability we have $\wh{L}_\mathrm{EZ}(f^\star) \ge \eps$. On the flip side, we know that by Hoeffding's inequality, $n=O(1/\eps^2)$ suffices for $\wh{L}_\mathrm{EZ}(f^\star) \le \eps$ with constant probability. This proves the proposition.
\end{proof}

\subsection{Main Result: Optimistic Rates for EasyLLP}
Despite the fact that the EasyLLP loss estimates only concentrate to their expectations at a rate of $O(1/\sqrt{n})$, it is possible improve upon the guarantee in Eq.~\eqref{eq:easyllp-slow} and show that the EasyLLP learning rule instantiated with the classification loss attains optimistic rates.

\begin{theorem}[Sample Complexity Bound for $\wh{f}_\mathrm{EZ}$]\label{thm:easyllp-opt-rate}
Let $\delta \in (0, 1)$. Fix any distribution $\cD$ and function class $\cF$, and let $L^\star \coloneqq \inf_{f \in \cF} \cL(f)$. With probability at least $1-\delta$ we have
\begin{align*}
\cL(\wh{f}_\mathrm{EZ}) \le L^\star + \wt{O} \bigg( \frac{k^2 (d + \log (1/\delta)) }{n} + \sqrt{\frac{L^\star \cdot k^2 \prn*{d+ \log (1/\delta)}}{n} } \bigg).
\end{align*}
\end{theorem}

The bound for EasyLLP is order-wise identical to that shown for the debiased square loss (\pref{thm:lossbagsq-opt-rate}).
In the agnostic setting, the guarantee in \pref{thm:easyllp-opt-rate} is worse in terms of dependence on $k$ compared to Eq.~\eqref{eq:easyllp-slow}; we leave sharpening the dependence on $k$ to future work. In the realizable setting, \pref{prop:easyllp-slow-est} and \pref{thm:easyllp-opt-rate} together show that there is a separation between the rate of estimation (which is necessarily $\Omega(1/\sqrt{n})$) and the rate of learning (which is $\wt{O}(1/n)$).

We sketch the proof ideas for \pref{thm:easyllp-opt-rate}, and we defer the full proof to \pref{app:proof-easyllp}.

\textbf{Proof Sketch.}
There is no hope for us to prove \pref{thm:easyllp-opt-rate} through the usual route of showing uniform convergence bound like the following:
\begin{align*}
    \text{for all}~f \in \cF, \abs{\cL(f) - \wh{L}_\mathrm{EZ}(f)} \le \wt{O} \bigg(\frac{k^2 \prn*{d + \log(1/\delta) }}{n} + \sqrt{\frac{\cL(f) \cdot k^2 \prn*{d+ \log(1/\delta)}}{n} }  \bigg),
\end{align*}
since the previous display directly contradicts \pref{prop:easyllp-slow-est} for $f^\star = \argmin_{f \in \cF} \cL(f)$.

Instead, we make the critical observation that the \emph{offset} empirical losses concentrate at the optimistic rate. Even though the learner does not know the identity of $f^\star$, it is still true that minimizing $\wh{L}_\mathrm{EZ}(\cdot)$ is the same as minimizing the offset loss $\wh{\Gamma}(\cdot, f^\star) \coloneqq \wh{L}_\mathrm{EZ}(\cdot) - \wh{L}_\mathrm{EZ}(f^\star)$, so we can equivalently think of the learning rule as minimizing $\wh\Gamma(\cdot, f^\star)$. The following lemma shows that one can prove an optimistic rate for the offset empirical losses. In the lemma, we use $\Gamma(f, f^\star) \coloneqq \cL(f) - \cL(f^\star)$ to denote the expected difference in classification error.

\begin{lemma}\label{lem:gamma-uc}
Let $f^\star = \argmin_{f \in \cF} \cL(f)$. Then with probability at least $1-\delta$ we have for all $f \in \cF$
\begin{align*}
    \abs*{\wh{\Gamma}(f, f^\star) - \Gamma(f, f^\star)} \le \wt{O} \bigg( \frac{k d + k^2 \log(1/\delta)}{n} + \sqrt{\frac{\cL(f) \cdot k^2 \prn*{d+ \log(1/\delta)}}{n} }  \bigg).
\end{align*}
\end{lemma}
\pref{lem:gamma-uc} says 
that while the EasyLLP empirical estimate $\wh{L}_\mathrm{EZ}$ is only $O(1/\sqrt{n})$ close to the true classification error, the estimate of the \emph{difference} in classification error with that of $f^\star$ is more accurate, as much of the fluctuations in the bag estimates $\wh{L}_\mathrm{EZ}(\cdot)$ gets canceled out by subtraction.

In light of \pref{lem:gamma-uc}, \pref{thm:easyllp-opt-rate} follows from standard approach of translating uniform convergence bounds to guarantees on the returned predictor $\wh{f}_\mathrm{EZ}$, see \pref{app:proof-theorem-easyllp} for more details.

\paragraph{Estimating the Marginal Label Proportion.} The EasyLLP learning rule requires knowledge of the marginal label proportion $p = \Pr[y=1]$. We demonstrate that the optimistic rate guarantee in \pref{thm:easyllp-opt-rate} can be attained without knowledge of $p$, instead splitting the dataset into half to estimate $\wh{p}$ from a separate dataset (this was claimed without proof by \cite{busa2023easy}).

\begin{corollary}[Sample Complexity Bound for $\wh{f}_\mathrm{EZ}$ with Sample Splitting]\label{corr:easyllp-unknownp} Let $\delta \in (0,1)$. Fix any distribution $\cD$ and function class $\cF$, and let $L^\star \coloneqq \inf_{f \in \cF} \cL(f)$. Then with probability at least $1-\delta$ the EasyLLP learning rule with sample splitting satisfies
\begin{align*}
    \cL(\wh{f}_\mathrm{EZ})
    &\le L^\star + \wt{O} \bigg( \frac{k^2 (d+\log(1/\delta))}{n}  + \sqrt{\frac{L^\star \cdot k^2  \prn*{d+\log(1/\delta)}}{n} } \bigg).
\end{align*}
\end{corollary}

The details of the sample splitting procedure and the proof of \pref{corr:easyllp-unknownp} are shown in \pref{app:sample-splitting}. In contrast to the debiased square loss learning rule, where we could estimate square loss and bias terms from the same dataset, here our analysis requires a separate dataset in order to estimate $\wh{p}$. However, we conjecture that sample splitting is not required for the guarantee in \pref{corr:easyllp-unknownp}.

\medskip

%% file: lower_bounds_main.tex
\section{Lower Bounds}\label{sec:main-lower-bounds}

The sample complexity bounds we prove in \pref{thm:prop-matching-realizable}, \ref{thm:lossbagsq-opt-rate}, and \ref{thm:easyllp-opt-rate} are optimal (up to log factors) in terms of the dependence on $d$, $n$, and $\log (1/\delta)$. However, the question remains of resolving the optimal dependence on the bag size $k$. 

For every function class $\cF$, we have the trivial lower bound on the minimax sample complexity of $n = \Omega\prn*{d \log (1/\delta) /(k\eps)}$ for the realizable setting and $n = \Omega\prn*{d \log (1/\delta)/(k\eps^2)}$ for the agnostic setting, since LLP with $n$ bags is only harder than supervised learning with $nk$ examples. In general, the $1/k$ dependence in the lower bound cannot be improved, as there are function classes for which it is tight. For example, consider the function class $\cF = \crl{f_0, f_1}$ where $f_i(x) = i$ for $i \in \crl{0,1}$. For this class $\cF$, observing the label proportion $\alpha$ allows us to compute the average instance-level classification loss over the bag for both $f_0$ and $f_1$. Therefore, LLP with $n$ bags is no harder than supervised learning with $nk$ labeled examples, so $\cF$ can be PAC learned with $\wt{O}(1/(k\eps))$ samples in the realizable setting and $\wt{O}(1/(k\eps^2))$ in the agnostic setting.  

For \emph{specific} function classes $\cF$, it is possible to improve the $1/k$ lower bound.

\begin{theorem}\label{thm:main-lower-bound}
For any $d \ge 3$, there exists a function class $\cF$ with $\VC(\cF) = d$ such that any learning rule for LLP that PAC learns $\cF$ for bag size $k \le O \prn{ 2^d /\log d }$ with $\eps \le 1/16$, and $\delta \le 1/15$ requires $\Omega\prn*{\tfrac{d}{\log k}}$  samples in the realizable setting and $\Omega\prn*{\max \prn*{\tfrac{d}{\log k}, \tfrac{d}{\sqrt{k} \eps^2}}}$ samples in the agnostic setting.
\end{theorem}

The proof of \pref{thm:main-lower-bound} can be found in \pref{app:lower_bounds}. In the realizable setting, \pref{thm:main-lower-bound} gives a stronger lower bound when the accuracy parameter $\eps$ is a constant; however it is open to show a lower bound which dominates the trivial one of $\Omega(d/(k\eps))$ for $\eps \to 0$. In the agnostic setting, the lower bound in \pref{thm:main-lower-bound} dominates the trivial one for all $\eps \le 1/16$.

%% file: experiments.tex
\section{Experiments}\arxiv{\label{sec:experiments}}\colt{\label{app:experiments}}

In this section we empirically evaluate the performance of the learning rules discussed herein and present results which illustrate the differences of the learning rules in practical implementations.

\paragraph{Learning Rule Implementations.} Since minimizing the 0-1 loss is computationally intractable, we consider algorithmic variants of LLP learning rules which minimize a surrogate loss using minibatch stochastic gradient descent (SGD). Fix a parameterized function class $\cF = \crl{f_\theta: \theta \in \Theta} \subseteq [0,1]^\cX$. For EasyLLP, we minimize the loss \eqref{eq:easyLLP-loss} with $\ell_\mathrm{ins}: \cY \times \bbR \to \bbR$ instantiated to be the square loss \ezsq{} or the log loss \ezlog. Directly minimizing the debiased square loss via SGD is not straightforward: for every gradient update, the second term in the loss requires computing $\wh{\En} f_\theta$ (i.e., the average prediction of $f_\theta$ over the entire dataset). We approximate it with an exponential moving average $\wh{v}$ with parameter $\beta \in (0,1)$. Given a minibatch $S = \crl{(B_1, \alpha_1), \cdots, (B_\ell, \alpha_\ell)}$, we perform the updates
\begin{align*}
    \wh{v}^{+} &= \beta \cdot \wh{v} + (1-\beta) \cdot \frac{1}{\ell k}\sum_{i=1}^\ell \sum_{j=1}^k f_\theta(x_{i,j}),\numberthis\label{eq:dsq-updates-v} \\ 
    \theta^{+} &= \theta - \eta \cdot \nabla_\theta \prn*{\frac{1}{\ell} \sum_{i=1}^\ell k \cdot \bigg( \frac{1}{k} \sum_{j=1}^k f_\theta(x_{i,j}) - \alpha_i \bigg)^2 - (k-1) \prn*{\wh{v}^{+} - \wh{p}}^2}. \numberthis\label{eq:dsq-updates-theta}
\end{align*}
When performing backpropagation for the gradients on the debiasing term $(k-1) \prn*{\wh{v}^{+} - \wh{p}}^2$, the backpropagation does not go through $\wh{v}$, only the summation over $f_\theta(x_{i,j})$. We refer to this as \dsq{}. For \dsq{}, \ezlog{}, and \ezsq{}{}, the label proportion $\wh{p}$ is estimated as the average of the training set labels.
Lastly, we also consider algorithms that minimize the proportion matching square loss \eqref{eq:square-prop-learning-rule} and log loss \eqref{eq:log-prop-learningrule}, which we refer to as \pmsq{} and \pmlog{} respectively.

\subsection{Comparison of LLP Learning Rules}
The goal of this experiment is to understand how these learning rules perform under various bag sizes and function classes.

\paragraph{Experimental Setup.} We adopt a similar experiment setup as in \citep{busa2023easy}. We run LLP algorithms on the MNIST odd vs.~even task and the CIFAR10 animal vs.~machine task (binary classification versions of the MNIST and CIFAR10). We consider 5 architectures: a linear model, small two layer NN (with 100 hidden units), a large two layer NN (with 1000 hidden units), a small CNN, and a large CNN. We experiment with bag sizes $k \in \crl{10, 100, 1000}$. For each dataset, model, bag size, and algorithm, we run 10 trials with different random seeds. To select the learning rate, we report the best average achieved test error for learning rates in $\crl{0.01, 0.005, 0.001, 0.0005, 0.0001}$. In all of our experiments we use the Adam optimizer \citep{kingma2014adam} with minibatches of size $1000$ and train for 100 epochs. For \dsq{}, we use the approximate version in Eqs.~\eqref{eq:dsq-updates-v}-\eqref{eq:dsq-updates-theta} with $\beta = 0.99$. Further details can be found in \pref{app:implementation-details}.

\begin{figure}[t]
    \centering
    \begin{subfigure}[t]{0.4\textwidth}
        \centering
        \includegraphics[clip, trim=0cm 0cm 1cm 1cm, width=\linewidth]{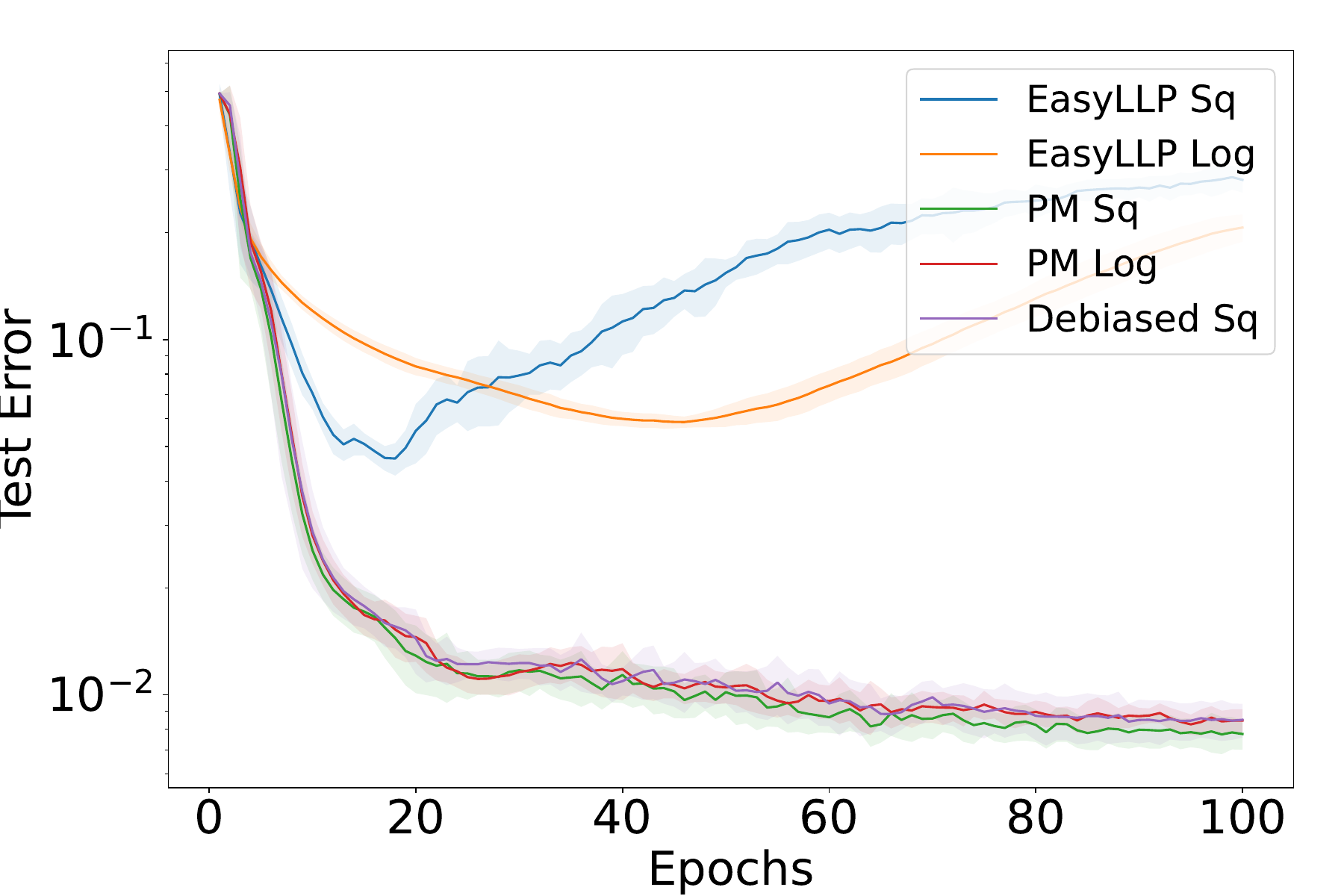} 
        \caption{MNIST Odd vs.~Even} \label{fig:mnist-tc}
    \end{subfigure}
    \hspace{2em}
    \begin{subfigure}[t]{0.4\textwidth}
        \centering
    \includegraphics[clip, trim=0cm 0cm 1cm 1cm,width=\linewidth]{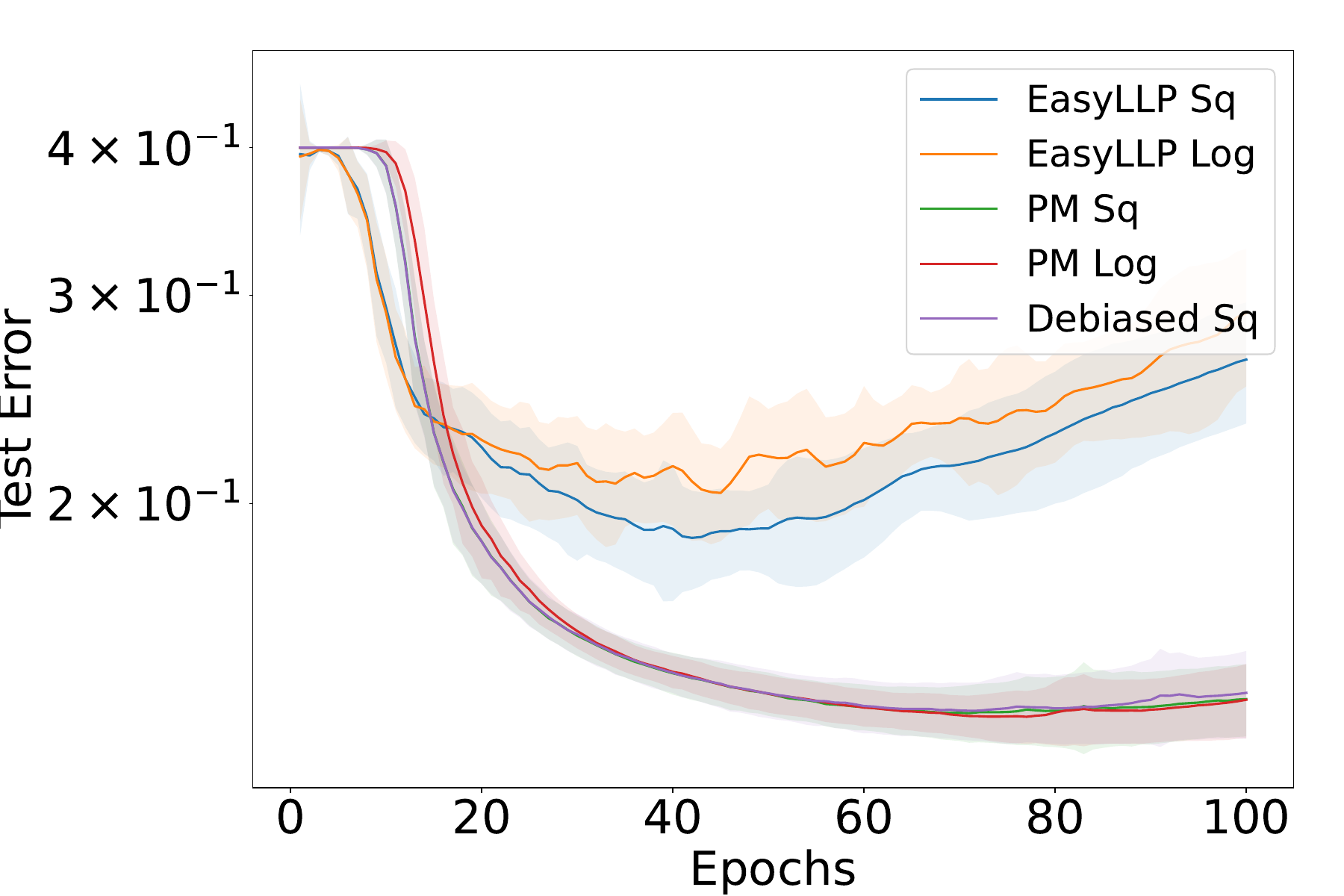} 
        \caption{CIFAR Animal vs.~Machine} \label{fig:cifar-tc}
    \end{subfigure}
    \vspace{-0.5em}
    \caption{Training curves of various algorithms for LLP, using the large CNN architecture and bag size $k=100$. One standard deviation confidence bands are plotted.}\label{fig:training-curves}
\end{figure}

\begin{figure}[t]
    \centering
    \begin{subfigure}[t]{0.31\textwidth}
        \centering
        \includegraphics[clip, trim=0.5cm 1cm 1cm 1cm, width=\linewidth]{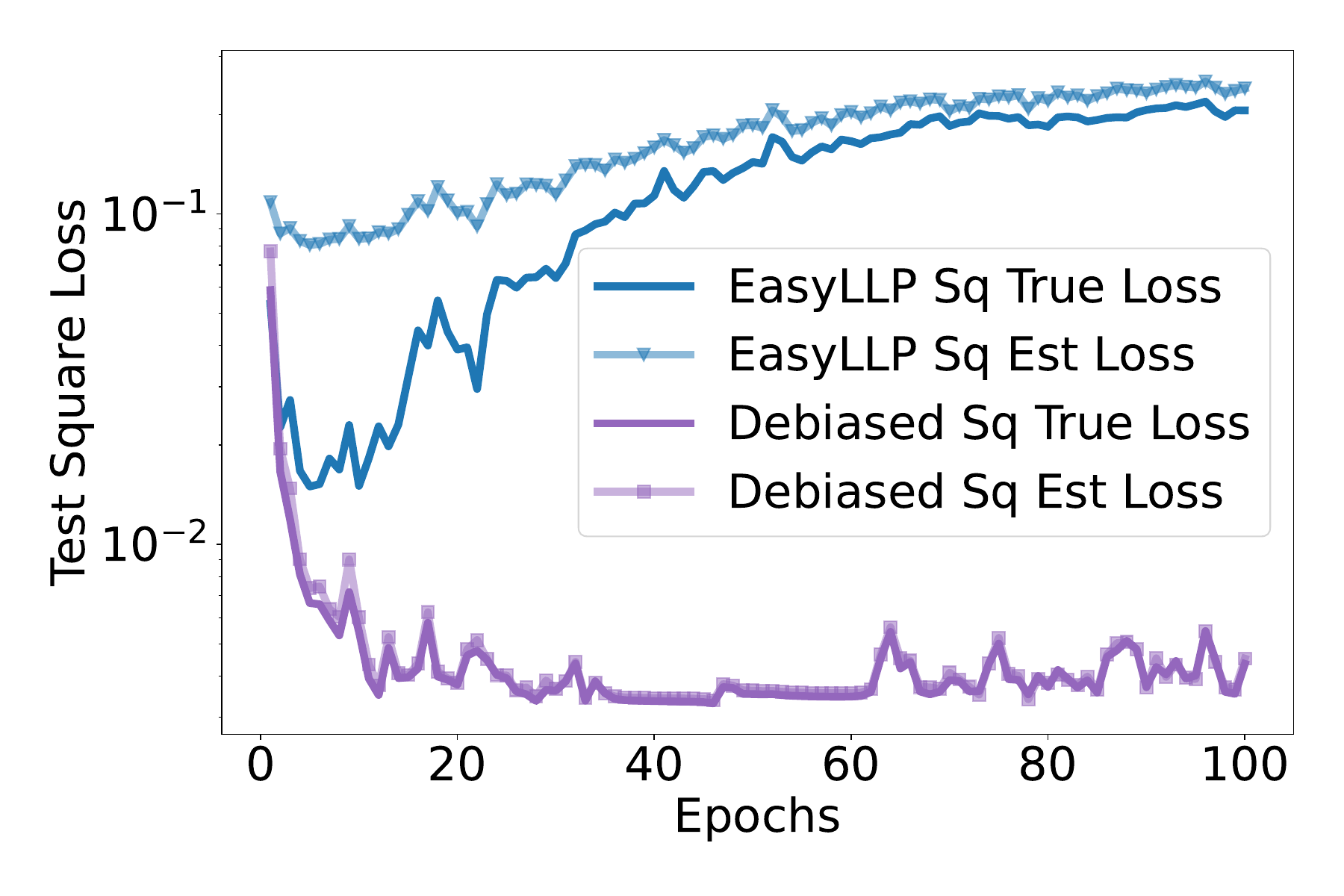} 
        \caption{Loss Estimates Tracking} \label{fig:loss-est}
    \end{subfigure}
    \hfill
    \begin{subfigure}[t]{0.31\textwidth}
        \centering
        \includegraphics[clip, trim=0.5cm 1cm 1cm 1cm, width=\linewidth]{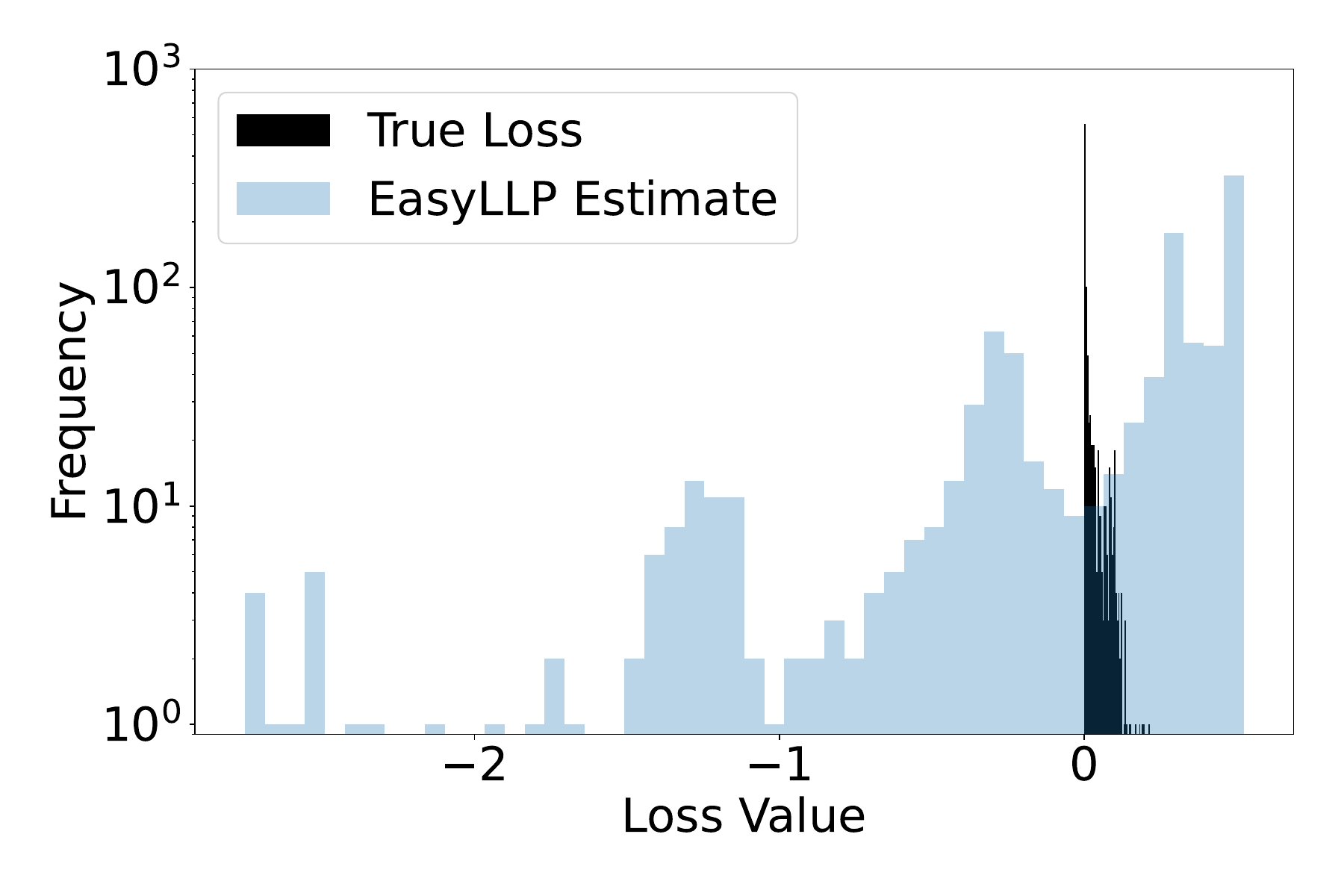} 
        \caption{\ezsq{} Loss Estimates} \label{fig:easylllp-hist}
    \end{subfigure}
    \hfill
    \begin{subfigure}[t]{0.31\textwidth}
        \centering
        \includegraphics[clip, trim=0.5cm 1cm 1cm 1cm, width=\linewidth]{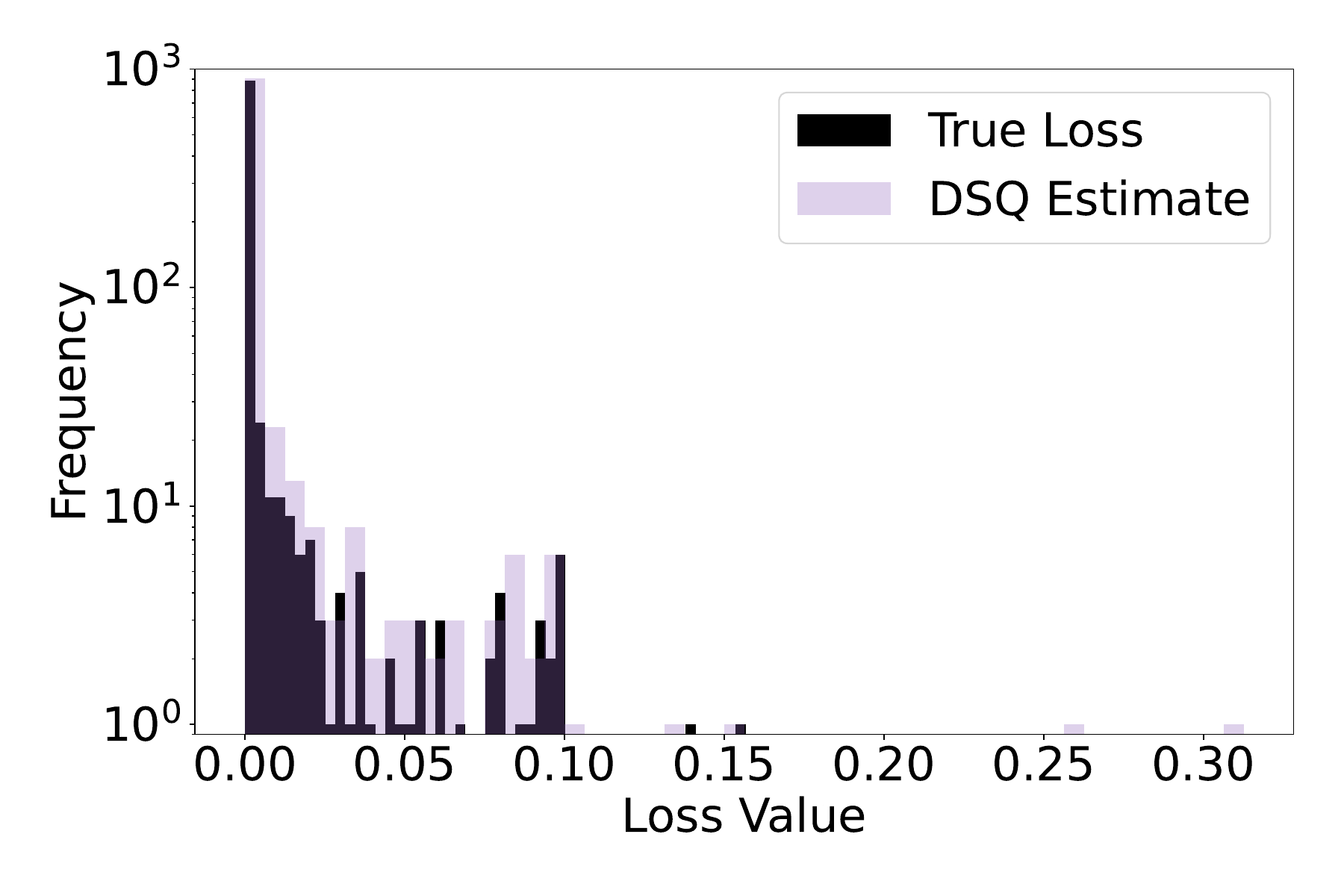} 
        \caption{\dsq{} Loss Estimates} \label{fig:dsq-hist}
    \end{subfigure}
    \vspace{-0.5em}
    \caption{Left: Loss estimates throughout training. We run a single trial of \ezsq{} and \dsq{} on MNIST Odd vs.~Even using the large CNN architecture, bag size $k=10$, and optimally chosen learning rate. Using the test set, we plot the averaged true square loss $\tfrac{1}{n_\mathrm{test}} \sum_{i=1}^{n_\mathrm{test}} \tfrac{1}{k} \sum_{j=1}^k \prn{f_\theta(x_{i,j}) - y_{i,j}}^2$ vs.~the estimated square loss $\tfrac{1}{n_\mathrm{test}} \sum_{i=1}^{n_\mathrm{test}} \wh{\ell}_\mathrm{est}(B_i, \alpha_i)$, where $\wh{\ell}_\mathrm{est}$ is either the \ezsq{}/\dsq{} loss estimate. Middle and Right: Histogram of true per-bag square losses $\crl{\tfrac{1}{k} \sum_{j=1}^k \prn{f_\theta(x_{i,j}) - y_{i,j}}^2}_{i=1}^{n_\mathrm{test}}$ and per-bag loss estimates $\crl{\wh{\ell}_\mathrm{est}(B_i, \alpha_i)}_{i=1}^{n_\mathrm{test}}$ for \ezsq{}/\dsq{} loss estimates at epoch 10.}\label{fig:loss-tracking}
\end{figure}

\paragraph{Discussion of Results.} In \pref{fig:training-curves} as well as \pref{tab:results-summary-mnist} and \ref{tab:results-summary-cifar} (in the Appendix), we see that the best algorithm is a tossup between \dsq{}, \pmsq{}, and \pmlog{}.\footnote{The results in \citep{busa2023easy} suggest that EasyLLP and the proportion matching baselines have similar performance; however, note that they only train for 20 epochs. We find that EasyLLP is competitive in early stages of training but eventually \dsq{}, \pmsq{}, and \pmlog{} outperform it.} In light of our discussion in \pref{sec:optimistic-rates-dsl}, this may be not that surprising, since the model classes we work with are expressive enough to minimize both the proportion matching loss and the bias, so we do not have the failure mode shown in \pref{prop:prop-matching-failure}. Furthermore, in \pref{fig:training-curves} we see that both versions of EasyLLP exhibits overfitting as training progresses, thus necessitating early stopping; this was also observed in \citep{busa2023easy}. We observe that  \dsq{}, \pmsq{}, and \pmlog{} also sometimes exhibits overfitting (see \pref{fig:cifar-tc}), but to a much less degree.

One hypothesis for why \dsq{} performs better than \ezsq{} is that is that the debiased square loss is a more accurate estimate of the true loss than the EasyLLP loss (see \pref{prop:easyllp-slow-est} and the discussion before it). In \pref{fig:loss-tracking}, we compare how well the estimated losses track the true square loss on the test set. Although we plot a single trial, we found that the behavior was similar across different random seeds. In \pref{fig:loss-est}, the \dsq{} loss closely tracks the true square loss, but \ezsq{} is consistently an \emph{overestimate} of the true square loss. (Interestingly, we observe that the ``shape'' of the loss curve is still preserved). An explanation for this phenomenon can be found in \pref{fig:easylllp-hist} and \pref{fig:dsq-hist}, where we plot the histogram of ground-truth per-bag square losses as well as the corresponding per-bag square loss estimates in the test set. The histogram of per-bag loss estimates for \dsq{} is quite similar to the histogram of ground-truth per-bag losses. On the other hand for \ezsq{}, the histogram of per-bag loss estimates has large variance and a heavy left tail. Since the randomness in the loss estimates is due to the bag generation procedure, it is likely that the bags with highly negative \ezsq{} loss estimates (i.e., taking values $< -2$) were not generated, so the sample mean overestimates the true square loss.

\begin{figure}
    \centering
    \begin{subfigure}[t]{0.48\textwidth}
        \centering
        \includegraphics[clip, trim=0cm 0.5cm 1cm 1cm, width=\linewidth]{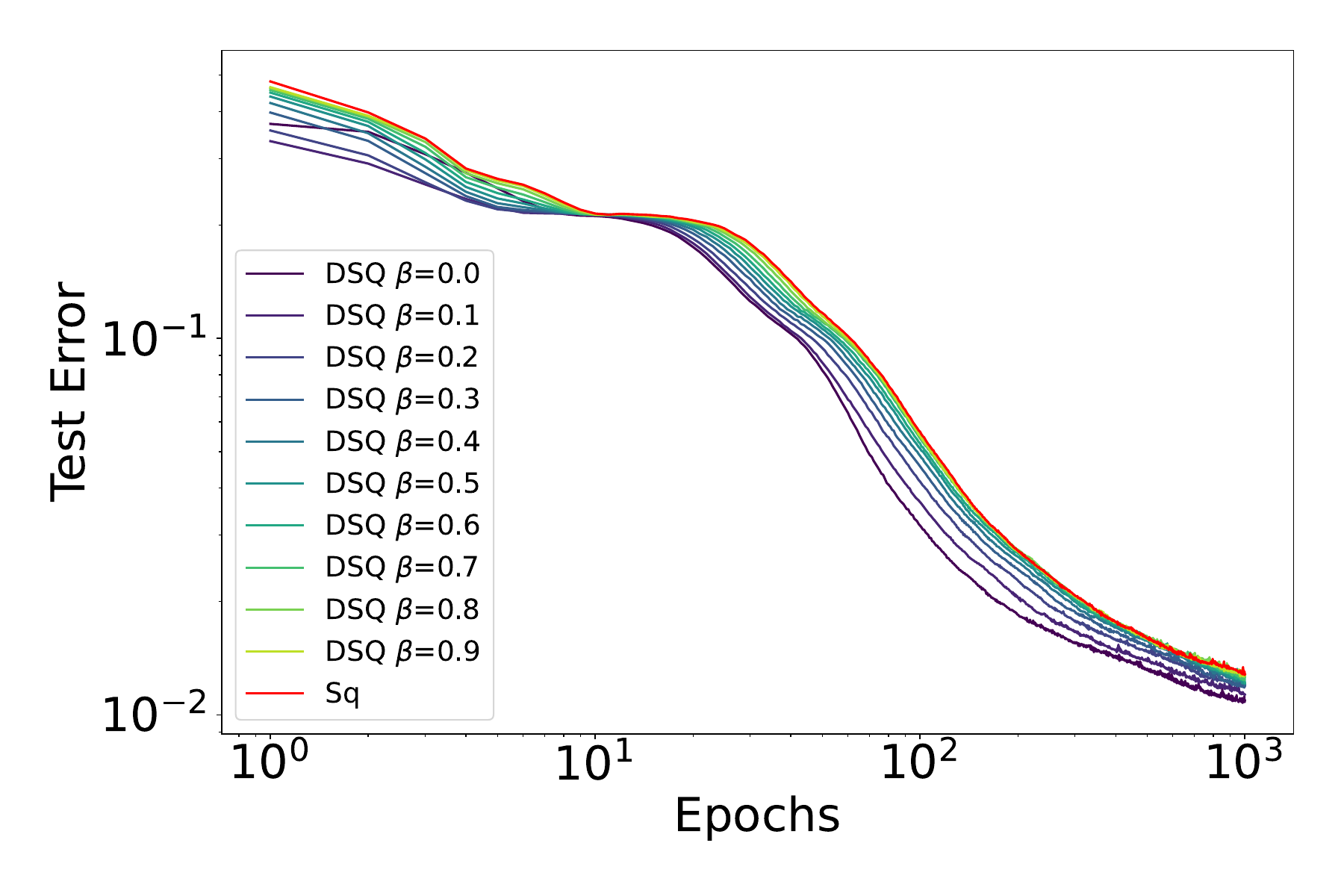} 
        \caption{MNIST Odd vs.~Even} \label{fig:mnist-dsq}
    \end{subfigure}
    \begin{subfigure}[t]{0.48\textwidth}
        \centering
        \includegraphics[clip, trim=0cm 0.5cm 1cm 1cm,width=\linewidth]{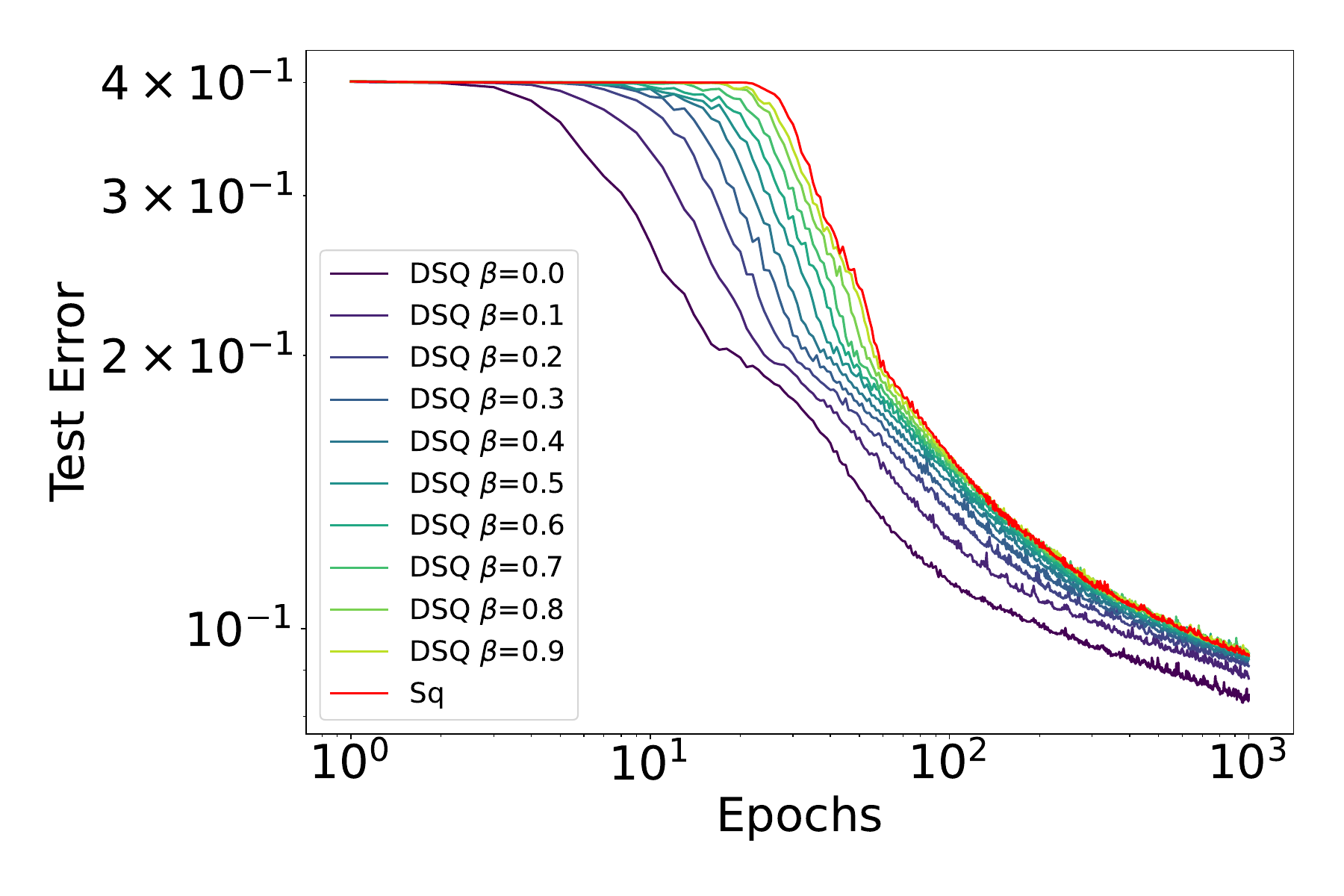} 
        \caption{CIFAR Animal vs.~Machine} \label{fig:cifar-dsq}
    \end{subfigure}
    \vspace{-0.5em}
    \caption{Training curves for \pmsq{} and \dsq{} with various $\beta$ for bag size $k=10$ on the small CNN architecture. We use a fixed learning rate of $0.001$ and run full-batch GD for 1000 epochs. Each line is an average over 10 trials with different random seeds.}\label{fig:full-batch}
\end{figure}

\subsection{Benefits of Debiasing for Optimization}
The previous results suggest that the performance of \pmsq{} and \dsq{} is similar, and that there might not be a need for debiasing in practice, as the failure mode in \pref{prop:prop-matching-failure} is arguably pathological. On the contrary, we present evidence that (accurate) debiasing can improve optimization. In \pref{fig:full-batch}, we compare the performance of \pmsq{} and \dsq{} with varying $\beta$ for full-batch GD. (We also investigated large batch SGD, but we found this effect was more striking with larger batches, so we only report the full-batch GD results.) Here, we can compute $\wh{\En} f_\theta$ in every epoch at no additional cost, so the debiasing term can be  exactly computed for every gradient update. 

We find that in early stages of training,  debiasing results in faster optimization. For CIFAR Animal vs.~Machine, GD on \pmsq{} seems to get stuck at a bad local minimum near initialization and only escapes after $\sim$25 epochs, but GD on \dsq{} is able to escape the local minimum much faster. Early in training, the model $f_\theta$ does not fit the label proportion $\wh{p}$ well, so the debiasing term $B(f_\theta)$ is large, and therefore the \pmsq{} loss will \emph{overestimate} the true instance-level square loss. On the contrary, \dsq{} uses more accurate estimates of the instance-level square loss, which is why we see benefits of debiasing. Later in training, as $B(f_\theta) \to 0$, the \pmsq{}/\dsq{} loss estimates are more similar, explaining why the different lines seem to converge to similar test errors.

%% file: discussion.tex
\section{Discussion}
Our work studies various learning rules for minimizing classification loss in the LLP framework. We show that EPRM attains fast rates under realizability, but EPRM and other proportion matching approaches can fail in the agnostic setting. For the debiased square loss and EasyLLP learning rules, we prove optimistic rate sample complexity bounds which are optimal (up to log factors) in terms of the dependence on $d$, $n$, and $\log(1/\delta)$ in both the realizable and agnostic settings. In addition, we investigate the optimal dependence on $k$ from the lower bound side. We also compare the empirical performance of gradient-based versions of these learning rules and demonstrate the benefits of debiasing for optimization\colt{ (deferred to \pref{app:experiments} due to space considerations)}.

For clarity of exposition, we focus on binary classification, but we note that both the debiased square loss and EasyLLP learning rules can be extended to the multi-label multi-class setting (as also observed in \cite{chen2023learning,busa2023easy}) by using one-hot encoding to write the label as a binary vector. In this way the learning task decouples into multiple binary classification tasks, and it would be straightforward to extend our analysis to this more general setting. 

Our work leaves open several future directions. The most immediate one is to resolve the optimal dependence on $k$. On the upper bound side, we believe that the dependence on $k$ can be improved in our analysis, and leave this to future work. Our optimistic rate results (\pref{thm:lossbagsq-opt-rate} and \ref{thm:easyllp-opt-rate}) are stated for a more general setting, and their proofs do not use the combinatorial structure of the LLP problem in the way that the proof of \pref{thm:prop-matching-realizable} does, thus hinting at a source of looseness. On the lower bound side, there is room to improve upon the construction in \pref{thm:main-lower-bound}. Other directions for future work include understanding the debiased square loss and EasyLLP learning rules in a more unified manner, designing a debiased variant of log loss for LLP, and studying the role of optimization in gradient-based algorithms for LLP.

%% file: failure_f_log.tex
\section{Failure of Proportional Log Loss in Agnostic Setting}\label{app:failure-flog}

\begin{proposition}
There exists a $\cF$ with $\VC(\cF) = 1$ and distribution $\cD$ such that as long as $n$, $k$, and $\delta$ satisfy the relationship $k \ge 18 \log (2n/\delta)$, with probability at least $1-\delta$, the learning rule $\wh{f}_\mathrm{LOG}$ is $1/3$-suboptimal.
\end{proposition}

Therefore, unless $n$ is exponentially large in the bag size $k$, the proportional log loss learning rule will return a suboptimal predictor.

\begin{proof}[Proof.]
We will use the same construction as in the proof of \pref{prop:prop-matching-failure}. For every bag $(B_i, \alpha_i)$ we will calculate the difference in proportional log loss for $f_1$ and $f_2$. For any $i \in [n]$, let $\beta_i \coloneqq \frac{1}{k}\sum_{j=1}^k \ind{x_{i,j} = x^{(1)}}$. Then we have
\begin{align*}
    \ell_\mathrm{LOG}(f_1, (B_i,\alpha_i)) &= - \alpha_i \log \beta_i - (1-\alpha_i) \log (1-\beta_i) \\
    \ell_\mathrm{LOG}(f_2, (B_i, \alpha_i)) &= - \alpha_i \log (1-\beta_i) - (1-\alpha_i) \log \beta_i.
\end{align*}
Therefore the difference between the two losses is
\begin{align*}
    \ell_\mathrm{LOG}(f_2, (B_i, \alpha_i)) - \ell_\mathrm{LOG}(f_1, (B_i,\alpha_i)) = (2\alpha_i - 1) \log \frac{\beta_i}{1-\beta_i}.
\end{align*}
We now show that with high probability, for all $i \in [n]$ we have $\ell_\mathrm{LOG}(f_2, (B_i, \alpha_i)) - \ell_\mathrm{LOG}(f_1, (B_i,\alpha_i)) \ge 0$.
By Hoeffding's inequality (\pref{thm:hoeffding-inequality}) and union bound we have
\begin{align*}
    \Pr[\exists i \in [n]: \alpha_i < 1/2 \text{ or } \beta_i < 1/2] \le 2n\exp\prn*{-\frac{k}{18}}.
\end{align*}
Thus as long as $k \ge 18 \log (2n/\delta)$, with probability at least $1-\delta$ we have $\ell_\mathrm{LOG}(f_2, (B_i, \alpha_i)) - \ell_\mathrm{LOG}(f_1, (B_i,\alpha_i)) \ge 0$ for all $i \in [n]$. This implies that $\wh{f}_\mathrm{LOG} = f_1$, the predictor which is $1/3$-suboptimal.
\end{proof}

%% file: optimistic_rates.tex
\section{Technical Background for Optimistic Rates}




\subsection{Uniform Convergence via Local Rademacher Complexity}
In this section, we establish technical several results on uniform convergence using local Rademacher complexities. Recall the worst-case Rademacher for a function class $\cG \subseteq \bbR^\cX$ for any $n \in \bbN$
\begin{align*}
    \rad_n(\cG) = \sup_{x_1, \cdots, x_n \in \cX^n} ~ \En_{\sigma} \brk*{\sup_{f \in \cG} \frac{1}{n} \sum_{i=1}^n \sigma_i f(x_j) },
\end{align*}
where $\crl*{\sigma_i}_{i=1}^n$ are i.i.d.~Rademacher random variables. For any $g$, we denote $\En[g]$ to denote its expectation and $\wh{\En}_n[g]$ to denote the empirical average over a sample $\crl{x_i}_{i=1}^n$ where $x_i$ are i.i.d.~drawn.


\begin{lemma}[Modified Version of Lemma 6.2 in \cite{bousquet2002concentration}]\label{lem:lem-bousquet-union}
Let $\cG \subseteq \bbR^\cX$ be a class of functions such that $\nrm{g}_\infty \le b$ for all $g\in \cG$ and let $\prn*{\cG_k}_{n \in \bbN}$ be a sequence of subsets of $\cG$ such that $\sup_{g\in \cG_k} \En[g^2] \le A + B \gamma_k$,
where $\gamma_k = b/2^k$ and $A, B > 0$ are constants.

Then for all $\delta \in (0,1)$ with probability at least $1-\delta$, for all $k \ge 0$ and $g\in \cG_k$:
\begin{align*}
    \abs{\En[g] - \wh{\En}_n[g]} \le 6  \mathfrak{R}_n(\cG_k) + \sqrt{\frac{2(A + B \gamma_k) \prn*{\log\frac{1}{\delta} + c\log\log \frac{b}{\gamma_k}}}{n}} + \frac{6b \prn*{\log\frac{1}{\delta} + c\log\log \frac{b}{\gamma_k}} }{n}.
\end{align*}
Here, $c > 0$ is an absolute constant.
\end{lemma}

\pref{lem:lem-bousquet-union} is a trivial modification of Lemma 6.2 in \cite{bousquet2002concentration}, so we omit the proof details.

\begin{definition}[Sub-Root Function] The function $\phi: \bbR_{\ge 0} \to \bbR$ is said to be a sub-root function if $\phi$ is non-negative, non-decreasing, not identically zero, and $\phi(r)/\sqrt{r}$ is non-increasing.
\end{definition}
    
\begin{assumption}[Variance-Expectation Bound]\label{ass:var-exp}
The function class $\cG \subseteq \bbR^\cX$ satisfies the following properties: (1) $\En[g] \ge 0$ for all $g \in \cG$ and (2) there exists constants $A, B \ge 0$ such that for all $g \in \cG$, $\En[g^2] \le A + B \En[g]$. 
\end{assumption}
\pref{ass:var-exp} is a slight generalization of the variance-expectation bound stated in Assumption 1.4 from \cite{bousquet2002concentration}; the difference is that we allow an additional constant offset $A \ge 0$.

\begin{theorem}[Modified Version of Theorem 6.2 in \cite{bousquet2002concentration}]\label{thm:modified-bousquet}
Let $\cG$ be a class of functions such that for all $g\in \cG$, $\lVert g\rVert_\infty \le b$ and $\cG$ satisfies \pref{ass:var-exp} with parameters $(A,B)$.

Let $\phi_n$ be a sub-root function such that
\begin{align*}
    \En_{\sigma} \brk*{\sup_{g: \wh{\En}_n[g^2] \le r} \frac{1}{n} \sum_{i=1}^n \sigma_i g(x_i)} \le \phi_n(r).
\end{align*}
Define $r_n^\star$ to be the largest solution of $\phi_n(r) = r$. For any $\delta > 0$, we have with probability at least $1-\delta$ for all $g \in \cG$:
\begin{align*}
    \abs{\En[g] - \wh{\En}_n[g]} \le C\prn*{b r^\star_n + \sqrt{r^\star_n (A+B\En[g])} + \sqrt{\frac{B \En[g]\prn*{\log\frac{1}{\delta} + c\log\log n}}{n}} + r_0 }.
\end{align*}
where $C>0$ is an absolute numerical constant and $r_0 \coloneqq\sqrt{\tfrac{2b^2 A \prn*{\log\tfrac{1}{\delta} + c\log\log n}}{n}} + \tfrac{22b^2 \prn*{\log\tfrac{1}{\delta} +  c\log\log n}}{n}$.
\end{theorem}

\begin{proof}[Proof.]
For all $k \in \bbN$ define $\gamma_k = b/2^k$. We define $\cG_k \coloneqq \crl{g \in \cG: \gamma_{k+1} < \En[g] \le \gamma_k }$, so that $\cG = \cup_{k \ge 0} \cG_k$.

By \pref{lem:lem-bousquet-union} and \pref{ass:var-exp}, we have with probability at least $1-\delta$ for all $k \ge 0$ and $g \in \cG_k$:
\begin{align*}
    \abs{\wh{\En}_n[g] - \En[g]} \le 8 \rad_n(\cG_k) + \sqrt{\frac{2 (A + B\gamma_k) \prn*{\log\frac{1}{\delta} + c\log\log \frac{b}{\gamma_k }}}{n}} + \frac{20b \prn*{\log\frac{1}{\delta} +  c\log\log \frac{b}{\gamma_k}}}{n}. \numberthis\label{eq:uc-bound}
\end{align*}
In addition, we can apply \pref{lem:lem-bousquet-union} to the squares of $g \in \cG_k$ to get that with probability at least $1-\delta$ for all $k \ge 0$ and every $g \in \cG_k$:
\begin{align*}
    \abs{\wh{\En}_n[g^2] - \En[g^2]} &\le 8 \rad_n(\cG_k^2) + \sqrt{\frac{2 b^2(A + B\gamma_k) \prn*{\log\frac{1}{\delta} + c\log\log \frac{b}{\gamma_k }}}{n}} + \frac{20b^2 \prn*{\log\frac{1}{\delta} +  c\log\log \frac{b}{\gamma_k}}}{n} \\
    &\le 16b \rad_n(\cG_k) + \sqrt{\frac{2 b^2(A + B\gamma_k) \prn*{\log\frac{1}{\delta} + c\log\log \frac{b}{\gamma_k }}}{n}} + \frac{20b^2 \prn*{\log\frac{1}{\delta} +  c\log\log \frac{b}{\gamma_k}}}{n},\numberthis\label{eq:uc-bound-squared}
\end{align*}
where the last inequality uses the fact that $x \mapsto x^2$ is $2b$-Lipschitz and centered at 0.

Now we condition on Eq.~\eqref{eq:uc-bound} and \eqref{eq:uc-bound-squared}, which happens with probability at least $1-2\delta$. By Eq.~\pref{eq:uc-bound-squared} and \pref{ass:var-exp}, for all $g \in \cG_k$,
\begin{align*}
    \wh{\En}_n[g^2] \le \prn*{A+B\gamma_k} + 16b \rad_n(\cG_k) + \sqrt{\frac{2 b^2(A + B\gamma_k) \prn*{\log\frac{1}{\delta} + c\log\log \frac{b}{\gamma_k }}}{n}} + \frac{20b^2 \prn*{\log\frac{1}{\delta} +  c\log\log \frac{b}{\gamma_k}}}{n}.
\end{align*}
Define the RHS of the previous display to be $U_k$. By definition of $\phi_n$, we know that 
\begin{align*}
    \rad_n(\cG_k) \le \sup_{x_1, \cdots, x_n} \En\sigma \brk*{\sup_{g: \wh{\En}_n[g^2] \le U_k} \frac{1}{n} \sum_{i=1}^n \sigma_i g(x_i) } \le \phi_n(U_k),
\end{align*}
so therefore
\begin{align*}
U_k \le (A+B\gamma_k) + 16 b\phi_n(U_k) + \sqrt{\frac{2b^2 (A + B\gamma_k) \prn*{\log\frac{1}{\delta} + c\log\log \frac{b}{\gamma_k }}}{n}} + \frac{20b^2 \prn*{\log\frac{1}{\delta} +  c\log\log \frac{b}{\gamma_k}}}{n}.
\end{align*}
Let us define $k_0$ to be the largest value such that $\gamma_{k_0+1} \ge \frac{b}{n}$. For all $k \le k_0$, we have $c \log\log \frac{b}{\gamma_k} \le c \log\log n$. Therefore
\begin{align*}
U_k &\le A+B\gamma_k + 16b \phi_n(U_k) + \sqrt{\frac{2b^2 (A + B\gamma_k) \prn*{\log\frac{1}{\delta} + c\log\log n}}{n}} + \frac{20b^2 \prn*{\log\frac{1}{\delta} +  c\log\log n}}{n} \\
&\le A+2B\gamma_k + 16b \phi_n(U_k)  +  \underbrace{ \sqrt{\frac{2b^2 A \prn*{\log\frac{1}{\delta} + c\log\log n}}{n}} + \frac{22b^2 \prn*{\log\frac{1}{\delta} +  c\log\log n}}{n} }_{\eqqcolon r_0}.
\end{align*}
If $U_k \ge r^\star_n$, then by definition of the sub-root function $\phi_n(U_k)/\sqrt{U_k} \le \phi_n(r^\star_n)/\sqrt{r^\star_n} = \sqrt{r^\star_n}$, so
\begin{align*}
    U_k \le 16b \sqrt{U_k r^\star_n} + A +2B\gamma_k + r_0 \le C \prn*{b^2 r^\star_n + A+B\gamma_k + r_0 } \eqqcolon r_n(\gamma_k)
\end{align*}
for some absolute constant $C > 0$. The last inequality holds by \pref{fact:ineq}. In addition if $U_k < r^\star_n$, the conclusion of the previous display trivially holds.

By Eq.~\eqref{eq:uc-bound}, we also know that for any $g \in \cG_k$
\begin{align*}
\En[g] &\le \wh{\En}_n[g] + 8 \rad_n(\cG_k) + \sqrt{\frac{2 B\gamma_k \prn*{\log\frac{1}{\delta} + c\log\log \frac{b}{\gamma_k }}}{n}} + r_0 \\
&\le \wh{\En}_n[g] + 8 \phi_n\prn*{r_n(\gamma_k)}  + \sqrt{\frac{2 B\gamma_k \prn*{\log\frac{1}{\delta} + c\log\log \frac{b}{\gamma_k }}}{n}} + r_0.
\end{align*}
By definition of $\cG_k$ and $\gamma_k$, we know that $\gamma_k \le 2 \En[g]$ so that
\begin{align*}
\En[g] &\le \wh{\En}_n[g] + 8 \cdot \phi_n\prn*{r_n(2 \En[g])} + \sqrt{\frac{4B \En[g] \prn*{\log\frac{1}{\delta} + c\log\log \frac{b}{\gamma_k }}}{n}} + r_0 \\
&\le \wh{\En}_n[g] + 8 \sqrt{r^\star_n} \cdot \sqrt{C \prn*{b^2 r^\star_n + A+B\En[g] + r_0} } + \sqrt{\frac{4B \En[g] \prn*{\log\frac{1}{\delta} + c\log\log \frac{b}{\gamma_k }}}{n}} + r_0\\
&\le \wh{\En}_n[g] + C'\prn*{ br^\star_n + \sqrt{r^\star_n (A+B\En[g])} }+ \sqrt{\frac{4B \En[g]\prn*{\log\frac{1}{\delta} + c\log\log n}}{n}} + 2r_0,
\end{align*}
for some absolute numerical constant $C' > 0$. When $k \ge k_0$, we have $\gamma_{k} \le \frac{b}{n}$, so the previous display also trivially holds.

Lastly, we can repeat the same argument with the function class $\cG' = \crl*{-g: g \in \cG}$ to get the two-sided bound. This concludes the proof of \pref{thm:modified-bousquet}.
\end{proof}

\subsection{Calculating the Complexity Radius for LLP Losses}
Now we present results which allow us to calculate the complexity radius $r^\star_n$ for function classes with bounded VC dimension.

\begin{lemma}[Refined Dudley's Inequality, Lemma A.1 from \cite{srebro2010optimistic}]\label{lem:refined-dudley}
For any function class $\cF: \cX \to \bbR$, 
\begin{align*}
    \rad_n(\cF) \le \inf_{\eta > 0} \crl*{4\eta + 12 \int_{\eta}^{\sup_{f \in \cF} \sqrt{\wh{\En}_n [f^2]}} \sqrt{\frac{\log \cN_2(\cF, \eps, n)}{n} }~d\eps}.
\end{align*}
\end{lemma}

\begin{lemma}[Theorem 2.14 in \cite{mendelson2003few}]\label{lem:covering-numbers} Let $\cF \subseteq \crl{0,1}^\cX$ with VC dimension $d$. Then
\begin{align*}
    \log \cN_2(\cF, \eps, n) \le d \log \prn*{4e^2 \log \frac{2e^2}{\eps}} + (2d)\cdot \log \prn*{\frac{1}{\eps}}.
\end{align*}
\end{lemma}

We next present a lemma which allows us to calculate the local Rademacher complexity for various loss functions for LLP. Fix any $\cF \subseteq \cY^\cX$ with VC dimension $d$. Consider any loss $\ell: [0,1]\times [0,1] \to \bbR$. Define the constrained loss class 
\begin{align*}
\cL_\ell(r) \coloneqq \crl*{\prn*{B, \alpha} \mapsto \ell\prn*{\frac{1}{k}\sum_{j=1}^k f(x_j), \alpha}: f \in \cF, \wh{\En}_n\brk*{\ell\prn*{\frac{1}{k}\sum_{j=1}^k f(x_j), \alpha}^2} \le r}.
\end{align*}
We also let $\cL_\ell$ denote the unrestricted loss class which contains all bag-level losses for $f \in \cF$.

\begin{lemma}\label{lem:phin-calc}
Let $\ell: [0,1]\times [0,1] \to \bbR$ be any $\lambda$-Lipschitz (in the first argument) bag loss. For any $n > d$ and any $r > 0$, we have
\begin{align*}
\mathfrak{R}_n(\cL_\ell(r)) \le C \sqrt{\frac{rd \log \prn*{\frac{\lambda n}{r}}}{n}}
\end{align*}
for some absolute numerical constant $C > 0$.
Furthermore, if we denote the sub-root function $\phi_n(r) \coloneqq C \sqrt{\frac{rd \log \prn{\frac{\lambda n}{r}}}{n}}$, and let $r^\star_n$ be the largest number such that $\phi_n(r) = r$, we have \begin{align*}
    r^\star_n \le O\prn*{\frac{d \log(\lambda n)}{n}}.
\end{align*}
\end{lemma}

\begin{proof}[Proof.]
We use \pref{lem:refined-dudley} applied to $\cL_\ell(r)$. This gives
\begin{align*}
    \mathfrak{R}_n(\cL_\ell(r)) \le \inf_{\eta > 0} \crl*{4\eta + 12 \int_{\eta}^{\sqrt{r}} \sqrt{\frac{\log \cN_2(\cL_\ell(r), \eps, n)}{n} }~d\eps}. \numberthis\label{eq:dudley-ub}
\end{align*}
From here we bound the covering numbers of the loss class in terms of the function class as
\begin{align*}
    \log \cN_2(\cL_\ell(r), \eps, n) \le \log \cN_2(\cL_\ell, \eps, n) \le \log \cN_2(\cF, \eps/\lambda, nk).
\end{align*}
The last inequality follows because for any $f, f_\eps$ and $(B_1, \alpha_1), \cdots, (B_n, \alpha_n)$:
\begin{align*}
\hspace{2em}&\hspace{-2em} \sqrt{\frac{1}{n} \sum_{i=1}^n \prn*{\ell\prn*{\frac{1}{k}\sum_{j=1}^k f(x_{i,j}), \alpha} - \ell\prn*{\frac{1}{k}\sum_{j=1}^k f_\eps(x_{i,j}), \alpha}  }^2} \\
&\le \lambda \cdot \sqrt{\frac{1}{n} \sum_{i=1}^n \prn*{\frac{1}{k} \sum_i f(x_{i,j}) - f_\eps(x_{i,j}) }^2} \le \lambda \sqrt{\frac{1}{nk} \sum_{i=1}^n  \sum_{j=1}^k \prn*{ f(x_{i,j}) - f_\eps(x_{i,j})} ^2},
\end{align*}
where we use the $\lambda$-Lipschitz property of $\ell$ as well as Jensen's inequality. Thus, an empirical $\ell_2$ cover of $\cF$ at scale $\eps/\lambda$ for $nk$ points implies an empirical $\ell_2$ cover of $\cL$ at scale $\eps$ for $n$ points.

From here, we can apply the covering number bound of \pref{lem:covering-numbers} to get
\begin{align*}
    \log \cN_2(\cF, \eps/\lambda, nk) \le d \log \prn*{4e^2 \log \frac{2e^2\lambda}{\eps}} + (2d)\cdot \log \prn*{\frac{\lambda}{\eps}} \le C d \log \prn*{\frac{\lambda}{\eps}},
\end{align*}
for some absolute numerical constant $C > 0$.

Now we can plug the covering number bound back into Eq.~\eqref{eq:dudley-ub} to get 
\begin{align*}
    \mathfrak{R}_n(\cL_\ell(r)) \le \inf_{\eta > 0} \crl*{4\eta + 12\sqrt{C} \int_{\eta}^{\sqrt{r}} \sqrt{\frac{d \log \frac{\lambda}{\eps}}{n} }~d\eps}.
\end{align*}
Choosing $\eta = \Theta\prn{\sqrt{rd/n}}$ gives
\begin{align*}
    \mathfrak{R}_n(\cL_\ell(r)) \le C' \sqrt{\frac{rd \log \prn*{\frac{\lambda n}{r}}}{n}},
\end{align*}
for some absolute numerical constant $C' > 0$. 
For the proof of the second part, it is easy to see that the solution to the equation $C' \sqrt{\tfrac{r^\star_nd \log \prn{\frac{\lambda n}{r^\star_n}}}{n}} = r$ must satisfy $r^\star_n \le C'' \cdot \tfrac{d \log(\lambda n)}{n}$ for some absolute constant $C'' > 0$.
This concludes the proof of \pref{lem:phin-calc}.
\end{proof}

%% file: debiased_sq_proof.tex
\section{Proof of \pref{thm:lossbagsq-opt-rate}}\label{app:proof-debiased-square}

\subsection{Notation and Preliminaries}
We define several quantities which will be used in the proof. The loss function $\wh{L}_\mathrm{DSQ}(f)$ is the difference between a proportion matching term and a debiasing term. For a given bag $(B,\alpha)$ and function $f$, we let $\ell_\mathrm{SQ}(f, (B,\alpha)) \coloneqq k \cdot \prn{\frac{1}{k} \sum_{j=1}^k f(x_j) - \alpha}^2$ and let
\begin{align*}
    \wh{L}_\mathrm{SQ}(f) \coloneqq \frac{1}{n} \sum_{i=1}^n \ell_\mathrm{SQ}(f, (B_i,\alpha_i)), \quad \text{and}\quad L_\mathrm{SQ}(f) \coloneqq \En_B \brk*{\ell_\mathrm{SQ}(f,(B,\alpha))}.
\end{align*}
In addition, we let
\begin{align*}
    \wh{B}(f) \coloneqq (k-1) \cdot \prn{\wh{\En} f - \wh{p}}^2, \quad \text{and} \quad B(f) \coloneqq (k-1) \cdot \prn{\En f - p}^2.
\end{align*}
Written in this notation, we have $\wh{L}_\mathrm{DSQ}(f) \coloneqq \wh{L}_\mathrm{SQ}(f) - \wh{B}(f)$. In addition, recall that we showed in \pref{sec:debiased-square-loss} that for any predictor $f$ we have $\cL(f) = L_\mathrm{SQ}(f) - B(f)$.

We also establish several elementary facts about $\wh{B}(f)$, $\wh{L}_\mathrm{SQ}$, and $\wh{L}_\mathrm{DSQ}$.
\begin{lemma}\label{lem:zb-relationship}The following statements are true for any predictor $f$ and dataset $\crl*{(B_i, \alpha_i)}_{i=1}^n$.
\begin{enumerate}
    \item $\wh{B}(f) \le \frac{k-1}{k} \cdot \wh{L}_\mathrm{SQ}(f)$.
    \item $\wh{L}_\mathrm{DSQ}(f) \ge 0$.
    \item $\wh{L}_\mathrm{SQ}(f) \le k \cdot \wh{L}_\mathrm{DSQ}(f)$.
\end{enumerate}
\end{lemma}
\begin{proof}[Proof.]
Fix any predictor $f$. For the first statement, using Jensen's inequality we can compute that
\begin{align*}
    \wh{B}(f) &= (k-1) \cdot \prn*{\frac{1}{n}\sum_{i=1}^n \frac{1}{k}\sum_{j=1}^k f(x_{i,j}) - \alpha_i}^2 \\
    &\le (k-1) \cdot \frac{1}{n} \sum_{i=1}^n \prn*{\frac{1}{k}\sum_{j=1}^k f(x_{i,j}) - \alpha_i}^2 = \frac{k-1}{k} \wh{L}_\mathrm{SQ}(f).
\end{align*}
The second statement is a simple consequence of the first statement. For the third statement we have
\begin{align*}
    \wh{L}_\mathrm{SQ}(f) = \wh{L}_\mathrm{DSQ}(f) + \wh{B}(f) \le \wh{L}_\mathrm{DSQ}(f) + \frac{k-1}{k} \cdot \wh{L}_\mathrm{SQ}(f),
\end{align*}
and rearranging yields the statement. This proves \pref{lem:zb-relationship}.
\end{proof}

\subsection{Showing Optimistic Rates} 
To prove \pref{thm:lossbagsq-opt-rate}, we separately prove optimistic rates for the square loss and debiasing terms, and then combine the guarantees. Consider the function class
\begin{align*}
    \cG \coloneqq \crl*{(B, \alpha) \mapsto k \cdot \prn*{\frac1k \sum_{j=1}^k f(x_j) - \alpha}^2: f \in \cF }.
\end{align*} 
For any function $g \in \cG$, we know that $\nrm{g} \le k$ and that $g$ is $k$-Lipschitz in the first argument; furthermore, by nonnegativity, $\cG$ satisfies \pref{ass:var-exp} with parameters $A=0$, $B = k$. Thus we apply \pref{thm:modified-bousquet} to $\cG$ to get that with probability at least $1-\delta$ for all $g\in \cG$,
\begin{align*}
    \abs{\En[g] - \wh{\En}_n[g]} \le C\prn*{k r^\star_n + \sqrt{r^\star_n (k\En[g])} + \sqrt{\frac{k \En[g]\prn*{\log\frac{1}{\delta} + c\log\log n}}{n}} +  \tfrac{k^2\prn*{\log\tfrac{1}{\delta} +  c\log\log n}}{n} }.
\end{align*}
where $r^\star_n$ is the critical radius and $C>0$ is an absolute numerical constant. Applying \pref{lem:phin-calc}, we can also calculate the critical radius as
\begin{align*}
    r^\star_n \le O\prn*{\frac{d \log(k n)}{n}},
\end{align*}
so therefore we have with probability at least $1-\delta$ for any $f \in \cF$:
\begin{align*}
    \abs*{ \wh{L}_\mathrm{SQ}(f) - L_\mathrm{SQ}(f)} \le \wt{O}\prn*{\frac{kd + k^2 \log\tfrac{1}{\delta}}{n} + \sqrt{\frac{ L_\mathrm{SQ}(f) \cdot k\prn*{d + \log\frac{1}{\delta} }}{n}}}. \numberthis\label{eq:zhat-convergence}
\end{align*}
Next, we show a uniform convergence bound which relates $\wh{B}(f)$ to $B(f)$. Consider the partition
\begin{align*}
    \cG^+ &\coloneqq \crl*{(x,y)\mapsto f(x) - y: f \in \cF, \En f \ge p},\\
    \cG^- &\coloneqq \crl*{(x,y)\mapsto y - f(x): f \in \cF, \En f < p}.
\end{align*}
The function class $\cG^+$ is $1$-Lipschitz, bounded in $[-1, +1]$, and satisfies \pref{ass:var-exp} with $A = 0$, $B = 1$. Applying \pref{thm:modified-bousquet} to $\cG^+$ and using \pref{lem:phin-calc} we get that with probability at least $1-\delta$ for any $f \in \cF$ such that $\En f \ge p$:
\begin{align*}
    \abs*{ \abs{\wh{\En}f - \wh{p}} - \abs{\En f - p}} \le \wt{O}\prn*{\sqrt{\frac{(\En f - p)\prn*{d + \log\frac{1}{\delta}}}{nk}} + \frac{d + \log\frac{1}{\delta}}{nk} }. \numberthis\label{eq:bhat-convergence}
\end{align*}
This implies
\begin{align*}
    (\wh{\En}f - \wh{p})^2 &\le (\En f - p)^2 + \wt{O}\prn*{\sqrt{\frac{(\En f - p)^{3}\prn*{d + \log\frac{1}{\delta}}}{nk}} + \frac{d + \log\frac{1}{\delta}}{nk} } \\
\Rightarrow \quad \wh{B}(f) &\le B(f) + \wt{O}\prn*{\sqrt{\frac{B(f)\prn*{d + \log\frac{1}{\delta}}}{n}} + \frac{d + \log\frac{1}{\delta}}{n} }. \numberthis\label{eq:bhat-convergence1}
\end{align*}
Likewise, for the reverse inequality we can also get that
\begin{align*}
    (\En f - p)^2 &\le (\wh{\En}f - \wh{p})^2  + \wt{O}\prn*{\abs{\wh{\En}f - \wh{p}} \cdot \sqrt{\frac{(\En f - p)\prn*{d + \log\frac{1}{\delta}}}{nk}} + \frac{d + \log\frac{1}{\delta}}{nk} } \\
    &\le (\wh{\En}f - \wh{p})^2  + \wt{O}\prn*{ \sqrt{\frac{(\En f - p)^3\prn*{d + \log\frac{1}{\delta}}}{nk}} + \frac{d + \log\frac{1}{\delta}}{nk} }  \\
\Rightarrow \quad B(f) &\le \wh{B}(f) + \wt{O}\prn*{\sqrt{\frac{B(f)\prn*{d + \log\frac{1}{\delta}}}{n}} + \frac{d + \log\frac{1}{\delta}}{n} }, \numberthis\label{eq:bhat-convergence2}
\end{align*}
where the second inequality uses the bound in Eq.~\eqref{eq:bhat-convergence}.

Combining Eq.~\eqref{eq:bhat-convergence1} and \eqref{eq:bhat-convergence2}, we get the two-sided bound for all $f\in \cF$ such that $\En f \ge p$:
\begin{align*}
    \abs{B(f) - \wh{B}(f)} \le \wt{O}\prn*{\sqrt{\frac{B(f)\prn*{d + \log\frac{1}{\delta}}}{n}} + \frac{d + \log\frac{1}{\delta}}{n} } \numberthis \label{eq:bhat-convergence3}
\end{align*} 

Following the same approach, we can show that the Eq.~\eqref{eq:bhat-convergence3} also holds for any $f \in \cF$ such that $\En f < p$, as the function class $\cG^-$ also is 1-Lipschitz, bounded in $[-1, +1]$ and satisfies \pref{ass:var-exp} with $A = 0, B=1$. Therefore we conclude that with probability at least $1-\delta$, Eq.~\eqref{eq:bhat-convergence3} holds uniformly for all $f \in \cF$.

Now we use Eqs.~\eqref{eq:zhat-convergence} and \eqref{eq:bhat-convergence3} to get
\begin{align*}
    \cL(\wh{f}_\mathrm{DSQ}) &= L_\mathrm{SQ}(\wh{f}_\mathrm{DSQ}) - B(\wh{f}_\mathrm{DSQ}) \\
    &\le \wh{L}_\mathrm{SQ}(\wh{f}_\mathrm{DSQ}) - \wh{B}(\wh{f}_\mathrm{DSQ}) \\
    &\quad\quad\quad + \wt{O}\prn*{\frac{kd + k^2 \log\tfrac{1}{\delta}}{n} + \sqrt{\frac{ L_\mathrm{SQ}(\wh{f}_\mathrm{DSQ}) \cdot k\prn*{d + \log\frac{1}{\delta} }}{n}} + \sqrt{\frac{B(\wh{f}_\mathrm{DSQ})\prn*{d + \log\frac{1}{\delta}}}{n}} } \\
    &\le  \wh{L}_\mathrm{SQ}(\wh{f}_\mathrm{DSQ}) - \wh{B}(\wh{f}_\mathrm{DSQ}) + \wt{O}\prn*{\frac{kd + k^2 \log\tfrac{1}{\delta}}{n} + \sqrt{\frac{ \cL(\wh{f}_\mathrm{DSQ}) \cdot k^2\prn*{d + \log\frac{1}{\delta} }}{n}} } \\
    &\le \wh{L}_\mathrm{SQ}(f^\star) - \wh{B}(f^\star) + \wt{O}\prn*{\frac{kd + k^2 \log\tfrac{1}{\delta}}{n} + \sqrt{\frac{ \cL(\wh{f}_\mathrm{DSQ}) \cdot k^2\prn*{d + \log\frac{1}{\delta} }}{n}} } \\
    &\le \cL(f^\star) + \wt{O}\prn*{\frac{kd + k^2 \log\tfrac{1}{\delta}}{n} + \sqrt{\frac{ \cL(\wh{f}_\mathrm{DSQ}) \cdot k^2\prn*{d + \log\frac{1}{\delta} }}{n}} }.
\end{align*}
The third inequality uses the fact that for any predictor $f$, we have $k L_\mathrm{SQ}(f) + B(f) =  k \cL(f) + (k+1)B(f) \le (k^2 + k - 1) \cL(f)$. The fourth and fifth inequalities uses the optimality of $\wh{f}_\mathrm{DSQ}$ and $f^\star$ for the empirical and population minimization problems respectively.

Finally, using the inequality \pref{fact:ineq} we get that
\begin{align*}
    \cL(\wh{f}_\mathrm{DSQ}) &\le \cL(f^\star) + \wt{O}\prn*{\frac{kd + k^2 \log\tfrac{1}{\delta}}{n} + \sqrt{\frac{ \cL(\wh{f}_\mathrm{DSQ}) \cdot k^2\prn*{d + \log\frac{1}{\delta} }}{n}} } \\
    &\le \cL(f^\star) + \wt{O} \prn*{\frac{k^2\prn*{d+ \log\tfrac{1}{\delta}}}{n}  + \sqrt{\cL(f^\star) + \frac{kd + k^2 \log\tfrac{1}{\delta}}{n} } \cdot \sqrt{\frac{ k^2\prn*{d + \log\frac{1}{\delta} }}{n}}
    } \\
    &= \cL(f^\star) + \wt{O} \prn*{\frac{k^2\prn*{d+ \log\tfrac{1}{\delta}}}{n}  + \sqrt{\frac{ \cL(f^\star)\cdot k^2\prn*{d + \log\frac{1}{\delta} }}{n}}
    }
\end{align*}

This concludes the proof of \pref{thm:lossbagsq-opt-rate}.

%% file: easyllp_proof.tex
\section{Proofs for \pref{sec:easyllp-main-text}}\label{app:proof-easyllp}


\subsection{Offset Loss Class}
To analyze the performance of $\wh{f}_\mathrm{EZ}$, we consider an \emph{offset} version of the EasyLLP loss estimate. Specifically, let $f^\star \coloneqq \inf_{f\in \cF} \cL(f)$. We define the offset loss
\begin{align*}
    \wh{\Gamma}(f, f^\star) \coloneqq \wh{L}_\mathrm{EZ}(f) - \wh{L}_\mathrm{EZ}(f^\star) = \frac{1}{n} \sum_{i=1}^n \prn*{k (2\alpha_i - 2p) + (2p-1)} \cdot  \prn*{\frac{1}{k}\sum_{j=1}^k f^\star(x_{i,j}) - f(x_{i,j})}
\end{align*}
Moreover, we use $\Gamma(f, f^\star)$ to denote its expectation, and we have $\En \brk{\wh{\Gamma}(f, f^\star)} = \cL(f) - \cL(f^\star)$. Clearly, minimizing the original EasyLLP loss is equivalent to minimizing $\wh{\Gamma}(f, f^\star)$. 

We also define the (empirical) second moment of the $\Gamma$ function as
\begin{align*}
    \wh{\Gamma}^2(f, f^\star) \coloneqq \frac{1}{n} \sum_{i=1}^n \prn*{k (2\alpha_i - 2p) + (2p-1)}^2 \cdot  \prn*{\frac{1}{k}\sum_{j=1}^k f^\star(x_{i,j}) - f(x_{i,j})}^2, 
\end{align*}
and use $\Gamma^2(f, f^\star)$ to denote $\En \brk*{\wh{\Gamma}^2(f, f^\star)}$.

We show that the offset loss class
\begin{align*}
    \cG \coloneqq \crl*{(B, \alpha) \mapsto \prn*{k (2\alpha - 2p) + (2p-1)} \cdot  \prn*{\frac{1}{k}\sum_{j=1}^k f^\star(x_j) - f(x_j)}: f \in \cF}
\end{align*}
satisfies \pref{ass:var-exp}.

\begin{lemma}\label{lem:class-var-exp}
The function class $\cG$ satisfies \pref{ass:var-exp} with $A = 8k^2 \cL(f^\star)$ and $B = 4k^2$.
\end{lemma}

\begin{proof}[Proof.]
First we observe that because $f^\star = \argmin_{f \in \cF} \cL(f)$, we have $\Gamma(f, f^\star) \ge 0$ for all $f \in \cF$. For the variance bound we can compute that
\begin{align*}
    \Gamma^2(f, f^\star) &\le 4 \cdot \En \brk*{ \prn*{\sum_{j=1}^k f^\star(x_j) - f(x_j)}^2 }
    \le 4k^2 \En \brk*{ \prn*{f^\star(x) - f(x)}^2 } \\
    &= 4k^2 \En \brk*{ \prn*{f^\star(x) - y + y -  f(x)}^2 } \le 4k^2 \prn*{\cL(f) + \cL(f^\star)}  \\
    &= 8k^2 \cL(f^\star) + 4k^2 \Gamma(f, f^\star).
\end{align*}
The first inequality uses the fact that $\prn*{k (2\alpha - 2p) + (2p-1)} \in [-2k+1, 2k-1]$. The second inequality uses the independence of the $\crl{x_j}$ as well as Jensen's inequality. The third inequality uses the fact that $(f(x) - y)^2 = \ind{f(x) \ne y}$, and that the cross terms satisfy $(f^\star(x) - y)(y - f(x)) \le 0$.

This concludes the proof of \pref{lem:class-var-exp}.
\end{proof}

\subsection{Uniform Convergence for Offset Loss Class}

Now we are ready to prove a uniform convergence bound for the function class $\cG$.
\begin{lemma}[\pref{lem:gamma-uc}, restated]\label{lem:gamma-uc-apdx}
Let $f^\star = \argmin_{f \in \cF} \cL(f)$. Then with probability at least $1-\delta$ we have for all $f \in \cF$
\begin{align*}
    \abs*{\wh{\Gamma}(f, f^\star) - \Gamma(f, f^\star)} \le \wt{O} \prn*{\frac{k d + k^2 \log \frac{1}{\delta}}{n} + \sqrt{\frac{ \cL(f) \cdot k^2 \prn*{d+ \log \frac{1}{\delta}}}{n} } }.
\end{align*}
\end{lemma}

\begin{proof}[Proof.]
By \pref{lem:class-var-exp}, we know that the function class $\cG$ satisfies \pref{ass:var-exp} with $A = 8k^2 \cL(f^\star)$ and $B = 4k^2$. We also know that $\nrm{g}_\infty \le k$ for all $g\in \cG$. Therefore, we can apply \pref{thm:modified-bousquet} to get that with probability at least $1-\delta$ for all $f \in \cF$,
\begin{align*}
    \hspace{1em}&\hspace{-1em} \abs*{\wh{\Gamma}(f, f^\star) - \Gamma(f, f^\star)} \\
    &\le C\prn*{k r^\star_n + \sqrt{r^\star_n \cdot 2k^2\cL(f)} + \sqrt{\frac{k^2 \cL(f) \prn*{\log\frac{1}{\delta} + c\log\log n}}{n}} + \tfrac{k^2 \prn*{\log\tfrac{1}{\delta} +  c\log\log n}}{n} }. \numberthis\label{eq:gamma-uc}
\end{align*}
where $r^\star_n$ is the critical radius and $C>0$ is an absolute numerical constant. 

Since the function $(B, \alpha) \mapsto \prn*{k (2\alpha - 2p) + (2p-1)} \cdot  \prn*{\frac{1}{k}\sum_{j=1}^k f^\star(x_j) - f(x_j)}$ is $2k$-Lipschitz, we have by \pref{lem:phin-calc} that 
\begin{align*}
    r^\star_n = O\prn*{\frac{d \log(k n)}{n}}.
\end{align*}
Plugging this into Eq.~\eqref{eq:gamma-uc} we get the conclusion of \pref{lem:gamma-uc}.
\end{proof}

\subsection{Proof of \pref{thm:easyllp-opt-rate}}\label{app:proof-theorem-easyllp}
Now we will prove the final generalization bound for EasyLLP. Using \pref{lem:gamma-uc}, we get that
\begin{align*}
    \Gamma(\wh{f}_\mathrm{EZ}, f^\star) &\le \wh{\Gamma}(\wh{f}_\mathrm{EZ}, f^\star) + \wt{O} \prn*{\frac{k d + k^2 \log \frac{1}{\delta}}{n} + \sqrt{\frac{ \cL(\wh{f}_\mathrm{EZ}) \cdot k^2 \prn*{d+ \log \frac{1}{\delta}}}{n} }   } \\
    &= \wh{\Gamma}(\wh{f}_\mathrm{EZ}, f^\star) + \wt{O} \prn*{\frac{k d + k^2 \log \frac{1}{\delta}}{n} 
 + \sqrt{\frac{ \prn*{\Gamma(\wh{f}_\mathrm{EZ}, f^\star) + \cL(f^\star)} \cdot k^2 \prn*{d+ \log \frac{1}{\delta}}}{n} } } 
\end{align*}
In fact, we know that since $f^\star \in \cF$, we must have $\wh{\Gamma}(\wh{f}_\mathrm{EZ}, f^\star) \le \wh{\Gamma}(f^\star, f^\star) = 0$, so using \pref{fact:ineq} we can further upper bound this as
\begin{align*}
    \Gamma(\wh{f}_\mathrm{EZ}, f^\star) 
    &\le \wt{O} \prn*{ \frac{k^2 \prn*{d + \log \frac{1}{\delta} }}{n} 
 + \sqrt{\frac{ \cL(f^\star) \cdot k^2 \prn*{d+ \log \frac{1}{\delta}}}{n} }}.
\end{align*}
Plugging in the definition of $\Gamma$ proves the bound.

\subsection{Proof of \pref{corr:easyllp-unknownp}}\label{app:sample-splitting}
We describe the sample splitting version of the EasyLLP learning rule, which allows us to use an estimate $\wh{p}$ instead of the true marginal label proportion $p$. We assume that we are given a dataset of size $2n$, denoted $S = \crl{(B_i, \alpha_i)}_{i=1}^{2n}$. We split it randomly into two equally-sized parts $S$ and $S'$. In the proof, we will use $i\in[n]$ to index bags in $S$ and $i' \in [n]$ to index bags in $S'$.
\begin{enumerate} 
    \item Using $S$, estimate marginal label proportion $\wh{p} = \frac{1}{n}\sum_{i=1}^{n} \alpha_i$.
    \item Return $\wh{f}_\mathrm{EZ} \coloneqq \argmin_{f \in \cF} \frac{1}{n} \sum_{i'=1}^{n} \ell_\mathrm{EZ}\prn*{f, (B_{i'},\alpha_{i'})}$ using $S'$, where $\ell_\mathrm{EZ}\prn{\cdot, \cdot}$ is defined with the plug-in estimate $\wh{p}$ instead of $p$:
    \begin{align*}
        \ell_\mathrm{EZ}\prn*{f, (B,\alpha)} = \prn*{k(\alpha - \wh{p}) + \wh{p}} \cdot \prn*{1- \frac{1}{k} \sum_{j=1}^k f(x_j)} + \prn*{k(\wh{p} - \alpha) + (1-\wh{p})} \cdot \prn*{\frac{1}{k} \sum_{j=1}^k f(x_j)}.
    \end{align*}
\end{enumerate}
For any $q \in [0,1]$ let us define the quantity
\begin{align*}
    \wh{\Gamma}_{q}(f, f^\star) \coloneqq \frac{1}{n} \sum_{i'=1}^{n} \prn*{k (2\alpha_{i'} - 2q) + (2q-1)} \cdot  \prn*{\frac{1}{k}\sum_{j=1}^k f^\star(x_{i',j}) - f(x_{i',j})},
\end{align*}
which represents the empirical offset loss estimated on $S'$ if we plugged in the value $q$ for the marginal label proportion.

By \pref{lem:gamma-uc} we know that for any $f \in \cF$
\begin{align*}
\Gamma_p(f, f^\star) &\le \wh{\Gamma}_p(f, f^\star) + \wt{O} \prn*{\frac{k d + k^2 \log \frac{1}{\delta}}{n} + \sqrt{\frac{ \cL(f) \cdot k^2 \prn*{d+ \log \frac{1}{\delta}}}{n} } }.
\end{align*}
From here, we need to relate $\wh{\Gamma}_{p}(f, f^\star)$ to $\wh{\Gamma}_{\wh{p}}(f, f^\star)$. We can bound the difference as:
\begin{align*}
    \abs*{\wh{\Gamma}_{p}(f, f^\star) - \wh{\Gamma}_{\wh{p}}(f, f^\star)} \le (2k-2) \cdot \abs*{\wh{p} - p} \cdot 
    \Big\lvert \underbrace{\frac{1}{n}\sum_{i'=1}^{n} \frac{1}{k} \sum_{j=1}^k f^\star(x_{i',j}) - f(x_{i',j}) }_{\eqqcolon \Xi(f)} \Big\rvert.
\end{align*}
By Hoeffding's inequality (\pref{thm:hoeffding-inequality}), with probability at least $1-\delta$ over $S$, we have $\abs*{\wh{p} - p} \le \sqrt{\frac{2\log(2/\delta)}{nk}}$. Henceforth we condition on this event holding over $S$.

Now we prove a uniform convergence bound on $\abs{\Xi(\cdot)}$ for all $f \in \cF$. For any $f \in \cF$ we write that
\begin{align*}
    \abs{\Xi(f)} &= \abs*{ \frac{1}{n}\sum_{i'=1}^{n} \frac{1}{k} \sum_{j=1}^k f^\star(x_{i',j}) - f(x_{i',j}) } \le \sqrt{\frac{1}{nk}
    } \cdot \sqrt{\sum_{i'=1}^{n} \sum_{j=1}^k \prn*{f^\star(x_{i',j}) - f(x_{i',j})}^2}  \\
    &\le \sqrt{\frac{1}{nk}
    } \cdot \sqrt{\sum_{i'=1}^{n} \sum_{j=1}^k \prn*{f^\star(x_{i',j}) - y_{i',j}}^2 + \prn{y_{i',j}-f(x_{i',j})}^2} = \sqrt{\wh{L}(f^\star) + \wh{L}(f)}.
\end{align*}
Here, we use $\wh{L}(\cdot)$ to denote the empirical classification loss on $S'$. The first inequality follows by Cauchy-Schwarz. The second inequality uses the fact that for $a,b,c \in \crl{0,1}$ we have $(a-b)(b-c) \le 0$. The last equality follows because $\prn{f(x)-y}^2 = \ind{f(x) \ne y}$.
Now we use the standard uniform convergence guarantee: with probability at least $1-\delta$ over $S'$, for any $f \in \cF$:
\begin{align*}
    \abs{\cL(f) - \wh{L}(f)} \le \wt{O} \prn*{\cL(f) + \frac{d + \log\frac{1}{\delta}}{nk}}.
\end{align*}
So therefore with probability at least $1-\delta$ over $S'$ we have for all $f \in \cF$:
\begin{align*}
    \abs{\Xi(f)} \le \sqrt{\wh{L}(f^\star) + \wh{L}(f)} \le \wt{O}\prn*{ \sqrt{\cL(f) + \frac{d + \log\frac{1}{\delta}}{nk}} }.
\end{align*}
Thus with probability at least $1-2\delta$ over the draws of $S_1$ and $S_2$, we have for all $f \in \cF$:
\begin{align*}
    \abs*{\wh{\Gamma}_{p}(f, f^\star) - \wh{\Gamma}_{\wh{p}}(f, f^\star)} \le \wt{O}\prn*{\frac{d + \log\frac{1}{\delta}}{n} + \sqrt{\frac{\cL(f)\cdot k\log\frac{1}{\delta}}{n} }  } \numberthis\label{eq:p-vs-phat-gamma}
\end{align*}
Now we are ready to prove the final guarantee. Similar to the proof of \pref{thm:easyllp-opt-rate} we compute that
\begin{align*}
\Gamma_p(\wh{f}_\mathrm{EZ}, f^\star) &\le \wh{\Gamma}_p(\wh{f}_\mathrm{EZ}, f^\star) + \wt{O} \prn*{\frac{k d + k^2 \log \frac{1}{\delta}}{n} + \sqrt{\frac{ \cL(\wh{f}_\mathrm{EZ}) \cdot k^2 \prn*{d+ \log \frac{1}{\delta}}}{n} } }\\
&\le \wh{\Gamma}_{\wh{p}}(\wh{f}_\mathrm{EZ}, f^\star) + \wt{O} \prn*{\frac{k d + k^2 \log \frac{1}{\delta}}{n} + \sqrt{\frac{ \cL(\wh{f}_\mathrm{EZ}) \cdot k^2 \prn*{d+ \log \frac{1}{\delta}}}{n} }  } \\
&\le \wt{O} \prn*{\frac{k d + k^2 \log \frac{1}{\delta}}{n} + \sqrt{\frac{ \cL(\wh{f}_\mathrm{EZ}) \cdot k^2 \prn*{d+ \log \frac{1}{\delta}}}{n} }  },
\end{align*}
where the second line uses Eq.~\eqref{eq:p-vs-phat-gamma} and the last line follows from the fact that $\wh{\Gamma}_{\wh{p}}(\wh{f}_\mathrm{EZ}, f^\star) \le \wh{\Gamma}_{\wh{p}}(f^\star, f^\star) = 0$. From here, the proof concludes similarly as the proof of \pref{thm:easyllp-opt-rate} in \pref{app:proof-theorem-easyllp} by using \pref{fact:ineq} and plugging in the definition of $\Gamma_p(\cdot, f^\star)$.

%% file: lower_bounds.tex
\section{Proof of \pref{thm:main-lower-bound}}\label{app:lower_bounds}
First, we describe the construction, then separately prove the lower bound for both the realizable and agnostic settings.

\subsection{Construction}We define the instance space $\cX = \crl{0,1}^{2^d}$ and let $\cF = \crl{f_i: i \in [2^d]}$ where the function $f_i$ is defined as $f_i(x) = x[i]$. It is clear that $\VC(\cF) \le d$. Now we define a family of distributions $\cD_i$ for $i \in [2^d]$. For some parameter choice $\gamma \in [0, 1/2]$, each $\cD_i$ is defined as follows: the example $x \sim \unif(\crl{0,1}^{2^d})$ and $y = f_i(x)$ with probability $1/2 + \gamma$, $y= 1-f_i(x)$ with probability $1/2 - \gamma$. For any predictor $f$ the classification loss for distribution $\cD_i$ can be written as
\begin{align*}
    \cL_{\cD_i}(f) - \inf_{f' \in \cF} \cL_{\cD_i}(f') = 2\gamma \cdot \En_{x \sim \unif\prn{\crl{0,1}^{2^d}} } 
    \brk*{\ind{f(x) \ne x[i]}}.
\end{align*}
Furthermore, for any predictor $f$, as well as distributions $\cD_i$ and $\cD_j$ we have the separation condition
\begin{align*}
    \hspace{2em}&\hspace{-2em}\cL_{\cD_i}(f) - \inf_{f' \in \cF} \cL_{\cD_i}(f') + \cL_{\cD_j}(f) - \inf_{f' \in \cF} \cL_{\cD_j}(f') \\
    &= 2\gamma \cdot \En_{x \sim \unif\prn{\crl{0,1}^{2^d}} } 
    \brk*{\ind{f(x) \ne x[i]} + \ind{f(x) \ne x[j]}} \ge \gamma. \numberthis\label{eq:separation-condition}
\end{align*}

\subsection{Realizable Setting} For the realizable setting result, we use the construction with $\gamma = 1/2$. We claim that for any learning rule for LLP that PAC learns $\cF$, there exists a distribution $\cD_i$ for which it requires $n = \Omega\prn{d/\log k}$ samples in expectation. Let us define $\bar{\cD}$ to be the averaged distribution where one first draws $i \sim \unif([2^d])$ then samples the bag $(B,\alpha) \sim \cD_i$. Using the separation condition  of Eq.~\eqref{eq:separation-condition}, we invoke Fano's inequality (e.g., Lemma 3 of \cite{yu1997assouad}) to get 
\begin{align*}
    \inf_{\wh{f}} \sup_{i\in [2^d]} \En_{\cD_i} \brk*{\cL_{\cD_i}\prn*{\wh{f}}} \ge \frac{1}{4} \prn*{1 - \frac{n \cdot \frac{1}{2^d} \sum_{i=1}^{2^d} \mathrm{KL}\prn*{\cD_i \Vert \bar{\cD}} + \log 2}{d}} = \frac{1}{4} \prn*{1 - \frac{n \cdot \mathrm{KL}\prn*{\cD_1 \Vert \bar{\cD}} + \log 2}{d}}, \numberthis\label{eq:fano-lb}
\end{align*}
where the equality follows by symmetry of the distributions $\cD_i$.

From here we need to estimate $\mathrm{KL}\prn*{\cD_1 \Vert \bar{\cD}}$. By chain rule for KL divergence, we see that
\begin{align*}
    \mathrm{KL}\prn*{\cD_1 \Vert \bar{\cD}} &= \mathrm{KL}\prn*{\Pr_{\cD_1}[B] \Vert \Pr_{\bar{\cD}}[B]} + \En_{B \sim \cD_1 } \brk*{ \mathrm{KL}\prn*{\Pr_{\cD_1}[\alpha ~|~ B] \Vert \Pr_{\bar{\cD}}[\alpha ~|~ B]}} \\
    &=\En_{B \sim \cD_1 } \brk*{ \mathrm{KL}\prn*{\Pr_{\cD_1}[\alpha ~|~ B] \Vert \Pr_{\bar{\cD}}[\alpha ~|~ B]}} \\
    &= \En_{B \sim \unif(\crl{0,1}^{2^d}) } \brk*{ \mathrm{KL}\prn*{\Pr_{\cD_1}[\alpha ~|~ B] \Vert \Pr_{\bar{\cD}}[\alpha ~|~ B]}},
\end{align*}
since all of the $\cD_i$ have the same marginal over $\cX$. Fix a bag $B =\crl{x_1,\cdots, x_k}$, and let us define the vector $z = \frac{1}{k}\sum_{j=1}^k x_j \in [0,1]^{2^d}$. We calculate that
\begin{align*}
    \mathrm{KL}\prn*{\Pr_{\cD_1}[\alpha ~|~ B] \Vert \Pr_{\bar{\cD}}[\alpha ~|~ B]} &= \sum_{\alpha \in \crl{0, \tfrac1k, \cdots, 1}} \Pr_{\alpha \sim \cD_1} \brk*{\alpha ~|~ B} \cdot \log \frac{\Pr_{\alpha \sim \cD_1} \brk*{\alpha ~|~ B}}{\Pr_{\alpha \sim \bar{\cD}} \brk*{\alpha ~|~ B} } \\
    &= \log \frac{1}{\Pr_{\alpha \sim \bar{\cD}} \brk*{\alpha = z[1] ~|~ B} } \\
    &= \log \frac{1}{\frac{1}{2^d} + \frac{1}{2^d} \sum_{i > 1} \Pr_{\alpha \sim \cD_i} \brk*{ \alpha = z[1] ~|~ B}} \\
    &= \log \frac{1}{\frac{1}{2^d} + \frac{1}{2^d} \sum_{i > 1} \ind{z[i] = z[1]}} \\
    &\le \min \crl*{d, \log \frac{2^d}{\sum_{i > 1}   \ind{z[i] = z[1]}} }.
\end{align*}
The second line follows from the fact that once we fix $B$, the value of $\alpha = z[1]$ is deterministic under $\cD_1$.
Putting it together we get that, we get
\begin{align*}
    \mathrm{KL}\prn*{\cD_1 \Vert \bar{\cD}} \le \En_{B \sim \unif(\crl{0,1}^{2^d})} \brk*{ \min \crl*{d, \log \frac{2^d}{\sum_{i > 1}   \ind{z[i] = z[1]}} } }.
\end{align*}
Observe that $k \cdot z[i]$ is distributed as independent $\mathrm{Bin}(k,1/2)$ variables for all $i \in [2^d]$. Using \pref{lem:discrete_iid}, we have for all $i > 1$ that $\Pr_{B}\brk{ z[i] = z[1]} \ge 1/(k+1)$, since the random variable $z[i]$ takes at most $k+1$ values. Thus, applying Chernoff bounds we have for any $\delta \in (0,1)$,
\begin{align*}
    \Pr_B \brk*{ \sum_{i > 1}   \ind{z[i] = z[1]} \ge \prn*{1-\delta} \frac{2^d - 1}{k+1}} \ge 1 - \exp\prn*{-\frac{\delta^2}{2} \cdot \frac{2^d - 1}{k+1}}.
\end{align*}
Let us pick $\delta = \sqrt{2(k+1) \log d / (2^d-1)}$; by our assumption on $k$ we have $\delta \le 1/2$. This guarantees that the above event, call it $\cE$, happens with probability at least $1- 1/d$. Therefore, we get 
\begin{align*}
    \mathrm{KL}\prn*{\cD_1 \Vert \bar{\cD}} \le d \cdot \frac{1}{d} + \log \frac{2^d \cdot 2(k+1)}{2^{d}-1} \le 2 + \log (k+1). \numberthis\label{eq:kl-ub}
\end{align*}
Therefore, plugging in Eq.~\eqref{eq:kl-ub} into Eq.~\eqref{eq:fano-lb} we see that
\begin{align*}
    \inf_{\wh{f}} \sup_{i\in [2^d]} \En_{\cD_i} \brk*{\cL_{\cD_i}\prn*{\wh{f}}} \ge \frac{1}{4} \prn*{1 - \frac{n \prn*{2 + \log (k+1)} + \log 2}{d}}.
\end{align*}
For $n = C'd/(\log (k+1))$ where $C' > 0$ is a sufficiently small constant we have $\inf_{\wh{f}} \sup_{i\in [2^d]} \En_{\cD_i} \brk*{\cL_{\cD_i}\prn*{\wh{f}}} \ge \frac{1}{8}$. Fix any learning rule $\wh{f}$, and let $\cD_{i^\star}$ be the distribution which witnesses the supremum. We have
\begin{align*}
    \frac{1}{8} \le \En_{\cD_{i^\star}} \brk*{\cL_{\cD_{i^\star}}\prn*{\wh{f}}} \le \Pr_{\cD_{i^\star}} \brk*{\cL_{\cD_{i^\star}}\prn*{\wh{f}} > \frac{1}{16}} + \frac{1}{16} \cdot \Pr_{\cD_{i^\star}} \brk*{\cL_{\cD_{i^\star}}\prn*{\wh{f}} \le \frac{1}{16}}
\end{align*}
which implies that $\Pr_{\cD_{i^\star}} \brk*{\cL_{\cD_{i^\star}}\prn*{\wh{f}} > \frac{1}{16}} \ge 1/15$. We conclude that any learning rule which PAC learns $\cF$ with parameters $(\eps, \delta) = (1/16, 1/15)$ requires $n = \Omega(d/\log k)$ samples.

\subsection{Agnostic Setting}

For the agnostic setting result we use the construction with $\gamma = \eps$. Again, by Fano's inequality we get
\begin{align*}
    \inf_{\wh{f}} \sup_{i\in [2^d]} \En_{\cD_i} \brk*{\cL_{\cD_i}\prn*{\wh{f}} - \inf_{f' \in \cF} \cL_{\cD_i}(f') } &\ge \frac{\eps}{2} \prn*{1 - \frac{n \cdot \frac{1}{2^d} \sum_{i=1}^{2^d} \mathrm{KL}\prn*{\cD_i \Vert \bar{\cD}} + \log 2}{d}} \\
    &= \frac{\eps}{2} \prn*{1 - \frac{n \cdot \mathrm{KL}\prn*{\cD_1 \Vert \bar{\cD}} + \log 2}{d}}, \numberthis\label{eq:fano-lb-agnostic}
\end{align*}
where the equality follows by symmetry of the distributions $\cD_i$.

From here we need to estimate $\mathrm{KL}\prn*{\cD_1 \Vert \bar{\cD}}$. By chain rule for KL divergence, we see that
\begin{align*}
    \mathrm{KL}\prn*{\cD_1 \Vert \bar{\cD}} &= \En_{B \sim \unif(\crl{0,1}^{2^d}) } \brk*{ \mathrm{KL}\prn*{\Pr_{\cD_1}[\alpha ~|~ B] \Vert \Pr_{\bar{\cD}}[\alpha ~|~ B]}}.
\end{align*}

From here we can bound this in two ways. The first way is to use the data-processing inequality for KL divergence. Observe that in distribution $\cD_i$, the label proportion $\alpha$ is distributed as $\mathrm{Bin}(kz[i], 1/2 +\eps) + \mathrm{Bin}(k-kz[i], 1/2-\eps)$. Therefore, by the data processing inequality
\begin{align*}
    \mathrm{KL}\prn*{\Pr_{\cD_1}[\alpha ~|~ B] \Vert \Pr_{\bar{\cD}}[\alpha ~|~ B]} &\le \mathrm{KL}\prn*{\Pr_{\cD_1}[\alpha_\mathrm{clean} ~|~ B] \Vert \Pr_{\bar{\cD}}[\alpha_\mathrm{clean} ~|~ B]} \\
    &\le \min \crl*{d, \log \frac{2^d}{\sum_{i > 1}   \ind{z[i] = z[1]}} },
\end{align*}
where $\alpha_\mathrm{clean} = z[i]$ in distribution $\cD_i$. From here, the proof proceeds similarly as in the realizable setting. We get the lower bound of $\Omega(d/\log k)$, with no dependence on $\eps$.

Alternatively, we can directly calculate the bound on the KL divergence:
\begin{align*}
    \hspace{2em}&\hspace{-2em}\mathrm{KL}\prn*{\Pr_{\cD_1}[\alpha ~|~ B] \Vert \Pr_{\bar{\cD}}[\alpha ~|~ B]} \\
    &\le \frac{1}{2^d}\sum_{i=1}^{2^d} \mathrm{KL}\prn*{\Pr_{\cD_1}[\alpha ~|~ B] \Vert \Pr_{\cD_i}[\alpha ~|~ B]} \\
    &= \frac{2^d-1}{2^d} \cdot  \mathrm{KL}\prn*{\Pr_{\cD_1}[\alpha ~|~ B] \Vert \Pr_{\cD_2}[\alpha ~|~ B]} \\
    &\le \max\crl{kz[1] - kz[2],0} \cdot \mathrm{kl}\prn{1/2 + \eps \Vert 1/2- \eps } + \max\crl{kz[2] - kz[1],0} \cdot \mathrm{kl}\prn{1/2 - \eps \Vert 1/2 + \eps }  \\
    &\le \abs{z[1] - z[2]} \cdot O(k\eps^2).
\end{align*}
The first line uses the convexity of KL. The second line uses the symmetry of the distributions $\cD_i$. The third line follows because under $\cD_i$, $k\alpha \sim \mathrm{Bin}(k z[i], 1/2+\eps) + \mathrm{Bin}(k-k z[i], 1/2-\eps)$, then applying chain rule for KL divergence. The last line applies the bound on the KL divergence between two Bernoulli random variables.

Now we investigate the expected difference between $z[1]$ and $z[2]$. Note that both of these are independently distributed as $\mathrm{Bin}(k, 1/2)/k$. By Hoeffding's inequality, we see that for any $i \in [2^d]$,
\begin{align*}
    \Pr_{B} \brk*{ \abs*{z[i] - \frac{1}{2}} \ge \sqrt{\frac{\log (2k)}{2k}}} \le \frac{1}{k},
\end{align*}
so by union bound, with probability at least $1-2/k$ we have $\abs{z[1] - z[2]} \le \sqrt{\frac{2\log (2k)}{k}}$. Using this, we can compute the bound that
\begin{align*}
    \mathrm{KL}\prn*{\cD_1 \Vert \bar{\cD}} &\le O(k \eps^2) \cdot \En_{B} \brk*{ \abs{z[1] - z[2]}} \\
    &\le O(k\eps^2) \cdot \prn*{\frac{2}{k} + \sqrt{\frac{2\log (2k)}{k} } } \\
    &\le O(\eps^2) \cdot \prn*{2 + \sqrt{ 2k\log (2k) } }.
\end{align*}
Plugging the previous display into Eq.~\eqref{eq:fano-lb-agnostic}, we see that if $n = C'd/(\sqrt{k} \eps^2)$ where $C' > 0$ is a sufficiently small constant we have $\inf_{\wh{f}} \sup_{i\in [2^d]} \En_{\cD_i} \brk*{\cL_{\cD_i}\prn*{\wh{f}} -\inf_{f' \in \cF} \cL_{\cD_i}(f')} \ge \frac{\eps}{4}$. As with the realizable setting proof, we can translate this to a lower bound for PAC learning; we omit the details.

%% file: technical.tex
\section{Technical Lemmas}

\begin{fact}\label{fact:ineq}
For any $A, B, C \ge 0$ if $A \le B + C \sqrt{A}$ then $A \le B + C^2 + \sqrt{B}C \le 2B + 2C^2$.
\end{fact}

\begin{theorem}[Hoeffding's Inequality]\label{thm:hoeffding-inequality}
Let $Z_1, \cdots, Z_n$ be independent bounded random variables with $Z_i \in [a,b]$ for all $i \in [n]$. Then 
\begin{align*}
    \Pr\brk*{\abs*{\frac{1}{n}\sum_{t=1}^n Z_i - \En[Z_i]} \ge t} \le 2\exp \prn*{ \frac{-2nt^2}{(b-a)^2}}.
\end{align*}
    
\end{theorem}

\begin{theorem}[Paley-Zygmund]\label{thm:paley-zygmund}
Let $Z \ge 0$ be a random variable with finite variance. For any $\theta \in [0,1]$,
\begin{align*}
    \Pr\brk*{Z \ge \theta \cdot \En Z} \ge (1-\theta)^2 \frac{\En[Z]^2}{\En [Z^2]}.
\end{align*}
\end{theorem}

\begin{lemma}\label{lem:discrete_iid}
Let $X$ and $Y$ be two i.i.d.~discrete random variables which are supported on a set of size $k$. Then $\Pr \brk*{X = Y} \ge 1/k$.   
\end{lemma}

\begin{proof}
Let $\Omega$ be the support. The probability that $X=Y$ can be calculated as
\begin{align*}
    \Pr \brk*{X = Y} = \sum_{x \in \Omega} \Pr[X=x]^2 \ge \frac{1}{k} \cdot \prn*{\sum_{x \in \Omega} \Pr[X=x]}^2 = \frac{1}{k}.
\end{align*}
The inequality uses the fact that $\nrm{v}_1 \le \sqrt{k}\nrm{v}_2$ for any $v \in \bbR^k$.
\end{proof}

%% file: app_experiments.tex
\arxiv{\section{Experimental Details}\label{app:experiments}}

\subsection{Implementation Details}\label{app:implementation-details}
Our code can be found on GitHub at \url{https://github.com/GXLI97/llp_experiments}. All experiments were run on an NVIDIA RTX A6000 GPU using Tensorflow and Keras. 

We elaborate on the architectures used in our experiments.
\begin{itemize}
    \item The linear model has a single dense output layer with 1 unit and sigmoid activation.
    \item The small two layer NN comprises of a dense layer with 100 units and ReLU activation, followed by a dense output layer with 1 unit and sigmoid activation.
    \item The large two layer NN comprises of a dense layer with 1000 units and ReLU activation, followed by a dense output layer with 1 unit and sigmoid activation.
    \item The small CNN is the same architecture in \citep{busa2023easy}:
    \begin{itemize}
        \item Convolutional layer with 32 kernels of size $3 \times 3$ and ReLU activation.
        \item Max pooling layer with pool $2 \times 2$.
        \item Convolutional layer with 64 kernels of size $3 \times 3$ and ReLU activation.
        \item Max pooling layer with pool size $2 \times 2$.
        \item Flatten layer.
        \item Dropout layer with drop rate $0.5$.
        \item Dense output layer with 1 unit and sigmoid activation.
    \end{itemize}
    \item The large CNN is the same architecture in \citep{busa2023easy}:
    \begin{itemize}
        \item Convolutional layer with 32 kernels of size $3 \times 3$ and ReLU activation.
        \item Convolutional layer with 32 kernels of size $3 \times 3$ and ReLU activation.
        \item Max pooling layer with pool $2 \times 2$.
        \item Dropout layer with drop rate $0.25$.
        \item Convolutional layer with 64 kernels of size $3 \times 3$ and ReLU activation.
        \item Convolutional layer with 64 kernels of size $3 \times 3$ and ReLU activation.
        \item Max pooling layer with pool $2 \times 2$.
        \item Dropout layer with drop rate $0.25$.
        \item Flatten layer.
        \item Dense layer with 512 units and ReLU activation.
        \item Dropout layer with drop rate $0.5$.
        \item Dense output layer with 1 unit and sigmoid activation.
    \end{itemize}
\end{itemize}

\colt{\clearpage}
\subsection{Results Summary}\label{app:additional-results}
\newcommand{\STAB}[1]{\begin{tabular}{@{}c@{}}#1\end{tabular}}

\begin{table}[h] 
\centering
\arxiv{\footnotesize}
\begin{tabular}{c l ccccc} 
\toprule
 & & \ezlog{} & \ezsq{} &
\pmlog{} &
\pmsq{} &
\dsq{} \\
\midrule
\multirow{5}{*}{\STAB{\rotatebox[origin=c]{90}{$k=10$}}} &
Linear & 10.9 & 10.2 & \bf{9.47} & 9.51 & 9.50 \\
& Two Layer Small & 5.32 & 3.93 & \bf{1.56} & \bf{1.56} & \bf{1.56} \\
& Two Layer Large & 4.56 & 3.50 & 1.28 & \bf{1.18} & 1.19 \\
& CNN Small & 3.89 & 2.47 & 1.06 & 1.08 & \bf{1.02} \\
& CNN Large & 3.54 & 1.58 & 0.417 & \bf{0.416} & 0.422 \\
\midrule
\multirow{5}{*}{\STAB{\rotatebox[origin=c]{90}{$k=100$}}} &
Linear & 14.2 & 12.9 & \bf{11.8} & \bf{11.8} & \bf{11.8} \\
& Two Layer Small & 8.72 & 8.33 & \bf{3.77} & 3.84 & 3.80 \\
& Two Layer Large & 9.29 & 8.50 & 3.21 & \bf{2.80} & \bf{2.80} \\
& CNN Small & 6.89 & 5.82 & 2.11 & 2.12 & \bf{2.08} \\
& CNN Large & 5.86 & 4.63 & 0.826 & \bf{0.774} & 0.841 \\
\midrule
\multirow{5}{*}{\STAB{\rotatebox[origin=c]{90}{$k=1000$}}} &
Linear & 21.3 & 21.5 & \bf{19.9} & \bf{19.9} & \bf{19.9}\\
& Two Layer Small & 20.5 & 20.5 & \bf{20.1} & 20.2 & 20.2 \\
& Two Layer Large & 19.6 & 21.0 & \bf{18.2} & 18.5 & 18.4 \\
& CNN Small & 16.1 & 16.3 & \bf{14.1} & 14.4 & 14.3 \\
& CNN Large & 15.3 & 15.2 & \bf{13.1} & 13.7 & 13.7 \\
\toprule
\end{tabular}
\vspace{-1em}
\caption{MNIST Test Error (\%) for LLP Algorithms. Best error is reported in \bf{bold}.}\label{tab:results-summary-mnist}
\arxiv{\vspace{-0.5em}}
\end{table}

\begin{table}[h] 
\centering
\arxiv{\footnotesize}
\begin{tabular}{c l ccccc} 
\toprule
 & & \ezlog{} & \ezsq{} &
\pmlog{} &
\pmsq{} &
\dsq{} \\
\midrule
\multirow{5}{*}{\STAB{\rotatebox[origin=c]{90}{$k=10$}}} &
Linear & 19.2 & 18.7 & \bf{18.4} & \bf{18.4} & \bf{18.4} \\
& Two Layer Small & 16.3 & 15.8 & \bf{14.6} & 14.7 & 14.7 \\
& Two Layer Large & 18.0 & 16.5 & \bf{13.2} & 13.4 & 13.4 \\
& CNN Small & 10.5 & 10.2 & 8.30 & \bf{8.09} & 8.23 \\
& CNN Large & 10.5 & 10.2 & 8.19 & \bf{8.11} & 8.18 \\
\midrule
\multirow{5}{*}{\STAB{\rotatebox[origin=c]{90}{$k=100$}}} &
Linear & 22.1 & 21.0 & 20.5 & \bf{20.4} & \bf{20.4} \\
& Two Layer Small & 20.1 & 20.1 & \bf{19.3} & \bf{19.3} & \bf{19.3}\\
& Two Layer Large & 20.3 & 24.5 & \bf{19.5} & 19.7 & 19.7 \\
& CNN Small & 17.4 & 17.3 & 13.0 & 13.0 & \bf{12.9} \\
& CNN Large & 20.4 & 18.7 & \bf{13.2} & 13.3 & 13.4 \\
\midrule
\multirow{5}{*}{\STAB{\rotatebox[origin=c]{90}{$k=1000$}}} &
Linear & 24.0 & 45.7 & \bf{28.5} & \bf{28.5} & \bf{28.5} \\
& Two Layer Small & 37.2 & 49.8 & 36.8 & \bf{36.7} & \bf{36.7} \\
& Two Layer Large & 27.1 & 54.0 & \bf{35.1} & 37.9 & 37.4 \\
& CNN Small & 33.5 & 33.6 & 35.7 & 35.8 & \bf{35.6} \\
& CNN Large & 36.1 & 36.1 & 35.9 & 36.8 & \bf{36.6} \\
\toprule
\end{tabular}
\vspace{-1em}
\caption{CIFAR Test Error (\%) for LLP Algorithms. Best error is reported in \bf{bold}.}\label{tab:results-summary-cifar}
\vspace{-1.5em}
\end{table}